\documentclass[letterpaper, 10 pt, conference]{IEEEtran}
\usepackage{etex}
\IEEEoverridecommandlockouts  

\usepackage{multicol}
\usepackage{graphicx,amsfonts,latexsym}
\usepackage[english]{babel}
\usepackage{amsmath,amsfonts,amssymb,amsthm}  
\usepackage{latexsym}
\usepackage{graphicx}
\usepackage[noend]{algpseudocode}  
\usepackage{xspace}
\usepackage{framed}
\usepackage{algpseudocode}
\usepackage{array,multirow,graphicx}
\usepackage{eurosym,color}
\usepackage{bm}
\usepackage{rotating}
\usepackage[ruled,vlined,linesnumbered]{algorithm2e}  

\newcommand{\red}[1]{\xspace{\color{red}#1}\xspace}

\usepackage{amsthm}
\newtheoremstyle{mystyle}
  {}
  {}
  {\itshape}
  {}
  {\bfseries}
  {.}
  { }
  {\thmname{#1}\thmnumber{ #2}\thmnote{ (#3)}}
\theoremstyle{mystyle}

\newtheorem{theorem}{Theorem}

\newtheorem{definition}[theorem]{Definition}
\newtheorem{proposition}[theorem]{Proposition}

\newcommand{\bdmath}{\begin{dmath}}
\newcommand{\edmath}{\end{dmath}}
\newcommand{\beq}{\begin{equation}}
\newcommand{\eeq}{\end{equation}}
\newcommand{\bdm}{\begin{displaymath}}
\newcommand{\edm}{\end{displaymath}}
\newcommand{\bea}{\begin{eqnarray}}
\newcommand{\eea}{\end{eqnarray}}
\newcommand{\beal}{\beq \begin{array}{ll}}
\newcommand{\eeal}{\end{array} \eeq}
\newcommand{\beas}{\begin{eqnarray*}}
\newcommand{\eeas}{\end{eqnarray*}}
\newcommand{\ba}{\begin{array}}
\newcommand{\ea}{\end{array}}
\newcommand{\bit}{\begin{itemize}}
\newcommand{\eit}{\end{itemize}}
\newcommand{\ben}{\begin{enumerate}}
\newcommand{\een}{\end{enumerate}}

\newcommand{\calD}{{\cal D}}
\newcommand{\calE}{{\cal E}}

\newcommand{\calG}{{\cal G}}

\newcommand{\calV}{{\cal V}}

\newcommand{\etal}{\emph{et~al.}\xspace}

\newcommand{\M}[1]{{\bm #1}} 

\newcommand{\hide}[1]{}

\newcommand{\hiddenText}{{\color{gray} hidden text.}}
\newcommand{\hideWithText}[1]{\hiddenText}

\newcommand{\lone}{\ell_{1}}

\newcommand{\tran}{^{\mathsf{T}}}

\newcommand{\rank}[1]{\mathrm{rank}\left(#1\right)}

\newcommand{\inv}{^{-1}}

\newcommand{\zero}{{\mathbf 0}}
\newcommand{\eye}{{\mathbf I}}

\newcommand{\Real}[1]{ { {\mathbb R}^{#1} } }

\newcommand{\SOtwo}{\ensuremath{\mathrm{SO}(2)}\xspace}

\newcommand{\MM}{\M{M}}

\newcommand{\MR}{\M{R}}

\newcommand{\ML}{\M{L}}

\newcommand{\MX}{\M{X}}

\newcommand{\MZ}{\M{Z}}
\newcommand{\MSigma}{\M{\Sigma}}

\newcommand{\vv}{\M{v}}
\newcommand{\vt}{\M{t}}
\newcommand{\vxx}{\M{x}}

\newcommand{\vmu}{\M{\mu}}

\renewcommand{\ij}{_{ij}}

\newcommand{\sumalledges}{
     \displaystyle
     \sum_{(i,j) \in \calE}}

\newcommand{\gtwoo}{{\smaller\sf g2o}\xspace}
\newcommand{\dcs}{{\smaller\sf dcs}\xspace}

\newcommand{\cvx}{{\sf cvx}\xspace}

\newcommand{\noiset}{\vt^\epsilon}
\newcommand{\noiseR}{\MR^\epsilon}
\newcommand{\noiseTheta}{\theta^\epsilon}

\newcommand{\omegat}{\omega^t}
\newcommand{\omegaR}{\omega^R}
\newcommand{\oomegat}{w^t}
\newcommand{\oomegaR}{w^R}

\newcommand{\barx}{\bar{\vxx}}
\newcommand{\barR}{\bar{\MR}}
\newcommand{\bart}{\bar{\vt}}

\newcommand{\subject}{\text{s.t.}}
\newcommand{\hatMX}{\hat{\MX}}
\newcommand{\hatMZ}{\hat{\MZ}}

\newcommand{\hatMR}{\hat{\MR}}
\newcommand{\hatvt}{\hat{\vt}}

\newcommand{\constraints}{\substack{\vt_i \in \Real{2}  \\ \MR_i \in \SOtwo}}

\newcommand{\scale}{\kappa}
\newcommand{\mpipi}{(-\pi,+\pi]}
\newcommand{\zeropi}{[0,+\pi)}
\newcommand{\half}{\frac{1}{2}}
\newcommand{\ltwo}{\ell_2}
\newcommand{\PED}{\text{ExpPow}}
\newcommand{\vonMises}{\text{VonMises}}
\newcommand{\gaussian}{\text{Normal}}
\newcommand{\dirLaplace}{\text{DLaplace}}

\newcommand{\uniform}{\text{Uniform}} 

\newcommand{\res}{r}
\newcommand{\forFinal}[1]{}

\newcommand{\huber}{\text{h}}

\newcommand{\myParagraph}[1]{{\bf #1.}\xspace}

\usepackage{float}

\usepackage{hyperref}
\usepackage{graphicx}
\usepackage{epstopdf}
\usepackage{epsfig}

\newcommand{\cht}{\check{\vt}}
\newcommand{\chR}{\check{\MR}}
\newcommand{\tilt}{\tilde{\vt}}
\newcommand{\tilR}{\tilde{\MR}}
\newcommand{\noiseang}{\epsilon^{R}}
\renewcommand{\ML}{ML\xspace}
\newcommand{\PGO}{PGO\xspace}
\renewcommand{\red}[1]{#1}
\usepackage{tikz}

\usepackage{xr}

\title{\huge{Convex Relaxations for Pose Graph Optimization with Outliers}}

\author{Luca Carlone and Giuseppe C. Calafiore
\thanks{
L.\,Carlone is with the Laboratory for 
Information \& Decision Systems, Massachusetts Institute of Technology, Cambridge, USA, {\smaller \sf lcarlone@mit.edu}
}
\thanks{
G.\,Calafiore is with the Dipartimento di Elettronica e Telecomunicazioni, 
Politecnico di Torino, Torino, Italy, {\smaller \sf giuseppe.calafiore@polito.it}
}
}

\begin{document}

\maketitle

\begin{abstract}
\emph{Pose Graph Optimization} \red{involves} the estimation of a set of poses from pairwise measurements 
and provides a formalization for many problems arising in mobile robotics and geometric computer vision.
In this paper, we consider the case in which a subset of the measurements fed to pose graph 
optimization is spurious.
Our first contribution is to develop robust estimators that can cope with heavy-tailed measurement noise, 
hence increasing robustness to the presence of outliers.
Since the resulting estimators require solving nonconvex optimization problems, 
we further develop 
convex relaxations that approximately solve those problems via semidefinite programming.
We then provide conditions under which the proposed relaxations are exact. 
Contrarily to existing approaches, our convex relaxations do not rely on the availability 
of an initial guess for the unknown poses,  hence they are more suitable for setups in which such guess is 
not available (e.g., multi robot localization, recovery after localization failure). 
We tested the proposed techniques in extensive simulations,
and we show that some of the proposed relaxations are indeed
 \emph{tight} (i.e., they solve the original nonconvex problem exactly)
 and ensure accurate estimation in the face of a large number of outliers.  
\end{abstract}

\begin{tikzpicture}[overlay, remember picture]
\path (current page.north east) ++(-4,-0.2) node[below left] {
This paper has been accepted for publication in the IEEE Robotics and Automation Letters.
};
\end{tikzpicture}
\begin{tikzpicture}[overlay, remember picture]
\path (current page.north east) ++(-4.1,-0.6) node[below left] {
 Please cite the paper as:
L. Carlone and G. Calafiore,
``Convex Relaxations for Pose Graph
};
\end{tikzpicture}
\begin{tikzpicture}[overlay, remember picture]
\path (current page.north east) ++(-4.6,-1) node[below left] {
 Optimization with Outliers'', IEEE Robotics and Automation Letters (RA-L), 2018.
};
\end{tikzpicture}


\section{Introduction}
\label{intro}

Pose Graph Optimization (\PGO) consists in the estimation of a set of poses (i.e., rotations and translations) 
from pairwise relative pose measurements. This model often arises from a \emph{maximum likelihood} (\ML) approach to 
geometric estimation problems in robotics and computer vision. 
For instance, in the context of 
robot \emph{Simultaneous Localization and Mapping} (SLAM), \PGO is used to estimate the
 trajectory of the robot (as a discrete set of poses), which \red{in turn} allows 
 the reconstruction of 
a map of the environment, see, e.g.,~\cite{Cadena16tro-SLAMsurvey}.  
Another example is multi robot localization, in which 
multiple robots estimate their poses from pairwise 
 measurements, see~\cite{Aragues11icra}. 
Similarly, a variant of \PGO arises in \emph{Structure from Motion} (SfM) in computer vision,
as a workhorse to estimate the poses (up to an unknown scale) of the cameras observing a 3D scene. 

\PGO leads to a hard non-convex optimization problem, whose (global) solution is the 
\ML estimate for the unknown poses.
While a standard approach to solve \PGO was to use iterative 
nonlinear optimization methods, such as the Gauss-Newton method~\cite{Kaess12ijrr,Kuemmerle11icra} or
the gradient method~\cite{Grisetti09its,Olson06icra}, 
to obtain locally optimal solutions,
a very recent set of works \red{shows how to}  compute \emph{globally}  
optimal solutions to \PGO via 
convex relaxations~\cite{Carlone15icra-verify,Carlone15iros-duality3D,Carlone16tro-duality2D,Rosen16wafr-sesync}.
These works demonstrate that in the noise regimes encountered in practical applications, the 
(non-convex) \PGO problem exhibits \emph{zero duality gap}, which implies that it can 
be solved exactly via convex relaxations. 

An outstanding problem in \PGO \red{is how to make} the estimation robust to the 
presence of spurious measurements.
Standard \PGO assumes a nominal distribution for the measurement noise (e.g., Gaussian noise on translation measurements), and it produces 
largely incorrect estimates in presence of \emph{outliers}, i.e., 
measurements that move away from this nominal distribution.
This issue limits robust operation in practical applications in which 
the presence of
outliers is unavoidable. 
Outliers may be due to sensor failures, but 
they are more commonly associated to incorrect \emph{data association}, see~\cite{Cadena16tro-SLAMsurvey}.

\myParagraph{Related Work}
 We partition the existing literature into 
 \emph{outlier mitigation} and \emph{outlier rejection} techniques.
The former techniques estimate the poses while trying to reduce the influence of the outliers.
The latter explicitly include binary decision variables to establish whether a given 
measurement is an outlier or not.
Traditionally, outlier mitigation in SLAM and SfM relied on the 
use of robust M-estimators, see~\cite{Huber81}. 
For instance, the use of the Huber loss is a fairly standard choice in SfM, see~\cite{Hartley00}.
Along this line, 
Agarwal \etal, in~\cite{Agarwal13icra}, dynamically adjust the measurement covariances to 
reduce the influence of measurements with large errors.
Olson and Agarwal, in~\cite{Olson12rss}, use  a max-mixture distribution to accommodate multiple 
hypotheses on the noise distribution of a measurement. 
Casafranca~\etal, in~\cite{Casafranca13iros}, propose to minimize the $\lone$-norm of the residual 
errors and design an iterative scheme to locally solve the corresponding optimization. 
Lee~\etal, in~\cite{Lee13iros}, use expectation maximization. 
\red{Pfingsthorn and Birk~\cite{Pfingsthorn16ijrr} 
 model ambiguous measurements using hyperedges and multimodal Mixture of Gaussian constraints that augment the pose graph.}

Outlier rejection techniques aim at explicitly identifying the spurious measurements.
A popular technique is RANSAC, see~\cite{Fischler81}, in which subsets of the measurements are sampled 
in order to identify an outlier-free set.
S\"{u}nderhauf and Protzel, in~\cite{Sunderhauf12iros,Sunderhauf12icra},
propose to augment the pose graph optimization problem with latent binary variables that are responsible for 
deactivating outliers. 
Similar ideas appear in the context of robust estimation, 
e.g., 
the \textit{Penalized Trimmed Squares} estimator of Zioutas and Avramidis, \cite{Zioutas05amas}.
Latif \etal in~\cite{Latif12rss} and Graham \etal in~\cite{Graham15iros}
look for ``internally consistent'' constraints, which are in mutual agreement. 
Carlone \etal~in \cite{Carlone14iros} use $\lone$-relaxation to find 
a large set of mutually-consistent measurements.

A main drawback of the techniques mentioned above is that they rely on the availability of an initial estimate of the poses. 
This is problematic for two reasons. First, the performance of these techniques 
heavily depends on the initial guess and they perform poorly when the initial guess 
is noisy, as it happens in practice. Second, in many applications, the initial 
guess is simply not available (e.g., in multi robot localization). 
Recent work in computer vision attempts to overcome the need for an initial guess. 
In particular, Wang and Singer in~\cite{Wang13ima} provide a convex relaxation for robust rotation estimation.
In this paper, we extend~\cite{Wang13ima} to work on poses (rather than rotations), and consider a broader 
 set of robust cost functions.

\myParagraph{Contribution}
Our first contribution is to propose robust \PGO formulations that can cope with heavy-tailed 
measurement noise. We consider cost functions with unsquared $\ltwo$ norm, $\lone$ norm, and 
Huber loss, and, when possible, we provide a probabilistic justification in terms of \ML estimation. 

Since the resulting optimization problems are nonconvex, 
our second contribution is to provide a systematic way to derive convex relaxations 
for those problems, taking the form of semidefinite programs (SDP). 
The key advantage of our convex relaxations is that 
they do not require an initial guess, which makes them suitable for problem instances where such estimate is not available or unreliable. 

As a third contribution, we provide conditions under which the proposed relaxations are exact, 
as well as general bounds on the suboptimality of the relaxed solution.
Similarly to related works such as~\cite{Carlone16tro-duality2D,Rosen16wafr-sesync}, 
these conditions involve the rank of a matrix appearing in the semidefinite relaxation.

Finally, we test the performance of the proposed relaxations in extensive Monte Carlo simulations.
The experimental results show that a subset of the convex relaxations discussed in this paper 
are indeed tight, and work extremely well in simulated problem instances, showing insensitivity 
to a large amount of outliers, \red{and outperforming related techniques}. 
\red{The paper is tailored to 2D \PGO problems, while the approaches and the theoretical guarantees 
trivially extend to the 3D setup.}

\myParagraph{Notation} We use lower and upper case bold letters to denote vectors 
($\vv$) and matrices ($\MM$), respectively. 
Non-bold face letters are used for scalars ($j$) and function names ($f(\cdot)$). 
The identity matrix of size $n$ is denoted with $\eye_n$. An $m \times n$ zero 
matrix is denoted by $\zero_{m\times n}$. 
The symbols $\|\cdot\|_1$, $\|\cdot\|_2$, and $\|\cdot\|_F$ denote the $\ell_1$, $\ell_2$, and Frobenious 
norm, respectively.  
Given a random variable $x$, we use the notation $\calD(x;\beta_1,\beta_2,\ldots)$ to denote  
the probability density of $x$ with parameters $\beta_1,\beta_2,\ldots$, and we also write 
$x \sim \calD(\beta_1,\beta_2,\ldots)$.


\section{Pose Graph Optimization}
\label{sec:PGO}
\PGO estimates $n$ poses from $m$ relative pose measurements.
We consider a planar setup, in which each to-be-estimated pose $\vxx_i \doteq (\MR_i,\vt_i)$, $i=1,\ldots,n$,
 comprises a \emph{translation} vector $\vt_i \in \Real{2}$  
and a rotation matrix $\MR_i \in \SOtwo$.
For a pair of poses $(i,j)$, 
a relative pose measurement $\barx\ij \doteq (\barR\ij,\bart\ij)$, with $\bart\ij \in \Real{2}$ and 
$\barR\ij \in \SOtwo$,
 describes a noisy measurement of 
the relative pose between $\vxx_i$ and $\vxx_j$. 
Each measurement is assumed to be sampled from the following
\emph{generative  model}:
\begin{align}
\label{eq:noisyModel}
\bart\ij = \MR_i\tran (\vt_j - \vt_i ) + \noiset\ij, \quad \quad \barR_{ij} = \MR_i\tran \MR_j \noiseR_{ij}
\end{align}
where $\noiset\ij \in \Real{2}$ 
and $\noiseR_{ij} \in \SOtwo$ represent translation and rotation measurement noise, respectively.

The problem can be elegantly modeled through graph formalism: 
each to-be-estimated pose \red{is associated to} a vertex (or node) in the vertex set ${\calV}$  of 
a directed graph, while each  measurement is associated to an edge in 
the edge set $\calE$ of the graph. We denote by $n\doteq |\calV|$  the number of nodes
and by $m\doteq |\calE|$  the number of edges.
The resulting graph, namely ${\calG}({\calV},{\calE})$, is usually referred to as a \emph{pose graph}. 

\subsection{Standard \PGO}
\label{sec:standardPGO}

Given the generative model~\eqref{eq:noisyModel}, standard approaches to \PGO 
formulate the pose estimation problem in terms of \ML estimation: 
the best pose estimate is the one maximizing the likelihood of the available measurements
 $\barx\ij, \forall (i,j) \in \calE$. 
 Since the expression of the measurement likelihood depends on the assumed distribution of the measurement noise,
 an ubiquitous assumption in this endeavour is that the translation noise $\noiset\ij$ 
 is distributed according to a zero-mean Normal distribution with given covariance $\MSigma^t\ij$, 
 i.e., $\noiset\ij \sim \gaussian(\zero,\MSigma^t\ij)$. 
For the rotation noise, it has been recently proposed to use the Von Mises distribution as noise model, as 
this leads to simpler estimators and it is easier to analyze, see~\cite{Carlone16tro-duality2D}.
In particular, it is assumed that $\noiseR\ij$ in~\eqref{eq:noisyModel} 
represents a rotation of a random angle $\noiseTheta$, distributed according 
to the Von Mises distribution:
\beq
\label{eq:wrappedGaussian}
\vonMises(\noiseTheta; \mu, \kappa) = 
c(\kappa) \exp(\kappa \cos(\noiseTheta - \mu) )
\eeq
with mean $\mu$, concentration parameter $\kappa$, and 
where $c(\kappa)$ is a normalization constant that is irrelevant for \ML estimation. 

 These choices of the noise model lead to the following standard \PGO formulation: 

 \begin{proposition}[Standard \PGO formulation]
 \label{prop:standardPGO}
 Given the generative model~\eqref{eq:noisyModel} and assuming that the 
 translation noise is $\noiset\ij \sim \emph{\gaussian}(\zero, \frac{1}{\omegat\ij} \eye_2)$  
 and that the rotation noise is $\noiseR\ij \sim \emph{\vonMises}(\zero, \omegaR\ij)$,  the 
 maximum likelihood estimate of the poses is the solution of the following minimization:
 \beq
\!\!\!\min_{\constraints}  \sumalledges \!\!
\omegat\ij \|\MR_{i}\tran(\vt_j\!-\!\vt_i)\!-\!\bart_{ij}\|^2_2 
\!+\! \frac{\omegaR_{ij}}{2} \|\MR_{i}\tran \!\MR_j -\barR_{ij}\|^2_F .
\label{eq:Squaredl2OnPosesPrimal}
\eeq
 \end{proposition}
%
A proof of this statement is given  in~\cite[Proposition 1]{Carlone16tro-duality2D}.
The estimation in~\eqref{eq:Squaredl2OnPosesPrimal} involves a nonconvex optimization 
problem, due to the nonconvexity of the set $\SOtwo$.
Surprisingly, 
recent results~\cite{Carlone16tro-duality2D,Rosen16wafr-sesync} show that 
one can still compute a globally optimal solution to~\eqref{eq:Squaredl2OnPosesPrimal}, 
whenever the measurement noise is reasonable, using 
convex relaxations. 

\section{Robust \PGO formulations}
\label{sec:1}


The ML estimator  in Proposition~\ref{prop:standardPGO} 
 is known to perform arbitrarily bad  in the presence of even a single outlier.
 This is due to the quadratic nature of the cost, for which
 measurements with large residual errors dominate the other terms.
 This in turn depends on the fact that the estimator assumes
a light-tailed noise distribution, i.e., the Normal distribution.
Therefore, a natural way to gain robustness is to change the noise 
assumptions to take into account the presence of measurements with larger 
errors, and use noise distributions with ``heavier-than-normal tails,'' see, e.g.,~\cite{Kotz}. 
In the rest of this section we consider three alternative noise models and 
their induced cost functions. 
Since the corresponding optimization problems remain nonconvex, 
we discuss how to relax these formulation to convex programs in Section~\ref{sec:relaxAndRound}-\ref{sec:2stage}.


\subsection{Least Unsquared Deviation: $\ell_2$-norm cost}
\label{sec:UnsquaredL2norm}


\red{We next introduce two distributions: 
the \emph{multivariate exponential power distribution}, which we later use 
to model the translation noise, and the \emph{directional Laplace distribution} 
which we use to model noise in \SOtwo.} 

\begin{definition}[Multivariate Exponential Power distribution]
\label{def:expPower}
For  a random variable $\vxx \in \Real{d}$, the \emph{multivariate exponential power distribution}
with parameters $\vmu \in \Real{n}$, $\MSigma \succ 0$, and $\beta > 0$ 
is defined as: 
\beq
\!\!\!\!\!\!\!\PED(\vxx;\mu,\Sigma,\beta) = 
c \exp\left( - \half 
[(\vxx\!-\!\vmu)\tran \MSigma\inv (\vxx\!-\!\vmu)]^\beta
\label{eq:mepd}
\right)
\eeq
where $c$ is a normalization constant, independent of $\vxx$.
\end{definition}

\begin{definition}[Directional Laplace Distribution]
\label{def:dirLaplace}
For  a random variable $\theta \in \mpipi$, the Directional Laplace Distribution
with mean $\mu \in \mpipi$ and scale parameter $\scale > 0$ is defined as:
\beq
\label{eq:dirLaplace}
\dirLaplace(\theta;\mu,\scale) = c \exp\left( - \scale
\left| \sin\left( \frac{\theta \!-\! \mu}{2}\right) \right| \right)
\eeq
where $c$ is a normalization constant, independent of $\theta$.
\end{definition}

Definition~\ref{def:dirLaplace} is a slight variation of the definition 
given in~\cite{Mitianoudis}, in that it considers $\mpipi$ as the angular domain, 
rather than $\zeropi$. 
Using Definitions~\ref{def:expPower}-\ref{def:dirLaplace}, we can now introduce our first robust estimator.

\begin{proposition}[Least Unsquared Deviation Estimator]
\label{prop:l2OnPosesPrimal}
Given the generative model~\eqref{eq:noisyModel} and assuming that the 
 translation noise is $\noiset\ij \sim \PED(\zero,\frac{1}{(\oomegat\ij)^2} \eye_2,\half)$  
 and the rotation noise is $\noiseR\ij \sim \dirLaplace(\zero, \oomegaR\ij)$, then the 
  \ML estimate of the poses is the solution of the following minimization problem:
  \beq
\!\!\!\min_{\constraints}  \sumalledges \!\!
\oomegat\ij \|\MR_{i}\tran(\vt_j\!-\!\vt_i)\!-\!\bart_{ij}\|_2 
\!+\! \frac{\oomegaR_{ij}}{\sqrt{2}} \|\MR_{i}\tran \!\MR_j - \barR_{ij}\|_F .
\label{eq:l2OnPosesPrimal}
\eeq
\end{proposition}

The derivation of the ML estimator of Proposition~\ref{prop:l2OnPosesPrimal} 
is straightforward, and proceeds by inspection, starting from the noise distributions~\eqref{eq:mepd}
and~\eqref{eq:dirLaplace}, and using the relation
 \begin{align}
\|\MR_{i}\tran \!\MR_j \!-\!\barR_{ij}\|_F = 2\sqrt{2}|\sin(\bar{\theta}\ij - \theta_j + \theta_i)|,
\end{align}
where $\bar{\theta}\ij$, $\theta_j$, and $\theta_i$ are the rotation angles corresponding to 
$\barR_{ij}$, $\MR_j$, and $\MR_i$, respectively.

The estimator~\eqref{eq:l2OnPosesPrimal} is similar to 
the standard \PGO estimator~\eqref{eq:Squaredl2OnPosesPrimal}, except for the fact that 
the $\ltwo$-norm and the Frobenius norm are unsquared. 
The intuition behind the use of~\eqref{eq:l2OnPosesPrimal} is that, 
since the cost is no longer quadratic, it does not favor 
measurements with large residuals.   

\subsection{Least Absolute Deviation: $\ell_1$-norm cost}
\label{sec:L1norm}

\red{In this section we consider a formulation in which the cost includes 
the $\lone$-norm rather than the Euclidean norm.}
%
\begin{definition}[Least Absolute Deviation Estimator]
\label{def:l1OnPosesPrimal}
The Least Absolute Deviation Estimator of the poses in the pose 
graph is the solution of the following minimization problem:
\beq
\!\!\!\min_{\constraints}  \sumalledges \!\!
\oomegat\ij \|\MR_{i}\tran(\vt_j\!-\!\vt_i)\!-\!\bart_{ij}\|_1 
\!+\! \frac{\oomegaR_{ij}}{2} \|\MR_{i}\tran \!\MR_j \!-\!\barR_{ij}\|_1 
\label{eq:l1OnPosesPrimal}
\eeq
\end{definition}
The formulation in Definition~\ref{def:l1OnPosesPrimal} 
adopts an $\lone$-norm penalty, which is known to be less sensitive 
to outliers, see, e.g.,~\cite[p.\ 10]{Rousseeuw81}, hence we expect this formulation 
to perform better in presence of spurious measurements. 
However, contrarily to Proposition~\ref{prop:l2OnPosesPrimal}, 
we currently only have a partial probabilistic justification for the 
cost~\eqref{eq:l1OnPosesPrimal}. More precisely, while 
it is known that the first term in the cost~\eqref{eq:l1OnPosesPrimal} follows from the assumption that the 
translation noise is distributed according to a Laplace distribution, see~\cite[p. 10]{Rousseeuw81}, 
the choice of the second term ($\|\MR_{i}\tran \!\MR_j \!-\!\barR_{ij}\|_1$) 
is currently arbitrary, and only justified 
by the symmetry w.r.t.\ the first term.

\subsection{Huber loss}
\label{sec:HuberFunction}

In this section we consider a popular M-estimator, based on the Huber loss function.
While this is a commonly used robust estimator in SLAM and SfM, 
our novel contribution is to provide a convex relaxation, which is discussed in   Section~\ref{sec:relaxAndRound}.

\begin{definition}[Huber Estimator]
\label{def:huberOnPosesPrimal}
The Huber Estimator of the poses in the pose 
graph is the minimizer of the following optimization problem:
\beq
\!\!\!\!\!\!\min_{\constraints}  \sumalledges \!\!
 \huber(\oomegat\ij \|\MR_{i}\tran(\vt_j\!-\!\vt_i)\!-\!\bart_{ij}\|_2)
+
 \huber(\oomegaR_{ij} \|\MR_{i}\tran \!\MR_j \!-\!\barR_{ij}\|_F) 
\label{eq:huberOnPosesPrimal}
\eeq
where $\huber(\cdot)$ is the Huber loss function, defined as:
\bea
\huber(x) = 
\left\{
\begin{array}{ll}
|x|^2 & |x|\leq 1 \\
2 |x| - 1 & \text{otherwise}.
\end{array}
\right.
\label{eq:huber}
\eea
\end{definition}
The Huber loss in Eq.~\eqref{eq:huber} is a quadratic function when the argument belongs to the 
interval $[-1,+1]$, while it is linear otherwise. 
Ideally, 
the inliers should fall in the quadratic region, where the 
Huber estimator behaves as the least squares estimator of Proposition~\ref{prop:standardPGO};
on the other hand, the outliers should ideally fall in the linear region, in which the 
Huber loss behaves as the least unsquared deviation estimator
of Proposition~\ref{prop:l2OnPosesPrimal}.
We note that while an alternative definition of the Huber loss, e.g.,~\cite[p. 619]{Hartley00},
includes an extra parameter $\delta$ that defines the size of the 
quadratic region ($[-\delta,+\delta]$), we are implicitly using 
the terms $\oomegat\ij$ and $\oomegaR_{ij}$ to define the region in which 
the cost is quadratic. For instance, in the first term of~\eqref{eq:huberOnPosesPrimal},
the Huber loss becomes quadratic when:
\beq
\oomegat\ij \|\MR_{i}\tran(\vt_j\!-\!\vt_i)\!-\!\bart_{ij}\|_2 \leq 1
\Leftrightarrow  \|\MR_{i}\tran(\vt_j\!-\!\vt_i)\!-\!\bart_{ij}\|_2  \leq  \frac{1}{\oomegat\ij} 
\eeq
i.e., the terms $\oomegat\ij$ and $\oomegaR_{ij}$  define the boundaries between 
the quadratic and the linear behavior. 

\forFinal{
We conclude this section by mentioning that the Huber loss has a 
probabilistic justification in terms of maximum likelihood estimation, and it 
follows from the assumption that the density of the measurement noise 
is a Huber density. 
}


\section{Convex Relaxations and Rounding Procedure}
\label{sec:relaxAndRound}

In this section we discuss a systematic approach for deriving convex relaxations for 
the problems~\eqref{eq:l2OnPosesPrimal},~\eqref{eq:l1OnPosesPrimal},~\eqref{eq:huberOnPosesPrimal}.
We refer to the approaches presented in this section as \emph{1-stage techniques} 
since they require the solution of a single convex program.
In Section~\ref{sec:2stage}, instead, we propose an alternative way to solve 
problems~\eqref{eq:l2OnPosesPrimal},~\eqref{eq:l1OnPosesPrimal},~\eqref{eq:huberOnPosesPrimal}, by decoupling rotation and translation estimation. 
 We refer to the latter as \emph{2-stage techniques}.
%
\subsection{Convex Relaxations of Robust PGO Formulations}

We consider each formulation in Section~\ref{sec:1}.

\subsubsection{Relaxation of the $\ell_2$-norm formulation}
Problem~\eqref{eq:l2OnPosesPrimal} is nonconvex and therefore hard to solve globally. 
The source of nonconvexity is the constraint $\M\MR_i \in \SOtwo$, while the 
cost can be made convex (in fact, quadratic) by rearranging the rotations $\M\MR_i$ 
and leveraging the invariance to rotation of the $\ltwo$ norm.
While in our previous work~\cite{Carlone16tro-duality2D} we used a quadratic 
reformulation of the cost, in the following we propose a more direct relaxation approach.

The first step towards our convex relaxation is 
to reparametrize Problem~\eqref{eq:l2OnPosesPrimal}. 
We rearrange the unknown poses $(\MR_i, \vt_i)$, with $i=1,\ldots,n$, 
into a single matrix $\MZ$:
\begin{gather}
\label{eq:MZ}
 \MZ \doteq [ \MR_1 \; \dots \; \MR_n \; \vt_1 \; \dots \; \vt_n]	 \in \Real{2\times3n}
\end{gather}

Moreover we define the following square matrix: 
\bea
\label{eq:MX}
\MX \doteq \MZ\tran \MZ \!\!\!
&\!\!\! = \!\!\!& \!\!\!
\left[\begin{array}{c c c|c c c} \MR_1\tran \MR_1 & \dots & \MR_1\tran \MR_n & \MR_1\tran \vt_1 & \dots & \MR_1\tran \vt_n  \\ 
 \vdots & \ddots & \vdots & \vdots & \ddots & \vdots \\ 
 \MR_n\tran  \MR_1 & \dots & \MR_n\tran \MR_n &  \MR_n\tran  \vt_1 & \dots & \MR_n\tran \vt_n \\
\hline \vt_1\tran \MR_1 & \dots & \vt_1\tran \MR_n & \vt_1\tran \vt_1 & \dots & \vt_1\tran \vt_n \\ 
 \vdots & \ddots & \vdots & \vdots & \ddots & \vdots \\ 
 \vt_n\tran  \MR_1 & \dots & \vt_n\tran \MR_n &  \vt_n\tran  \vt_1 & \dots & \vt_n\tran \vt_n \end{array} \right] \!\!\!\nonumber\\
\label{blockstruct}
&\!\!\!\doteq\!\!\!& \left[\begin{array}{c|c} \MX^{RR} & \MX^{Rt} \\  \hline \MX^{tR} & \MX^{tt}  \end{array} \right]   \in \Real{3n\times3n} \!\!\!
\eea
where we distinguished 4 block matrices within $\MX$, with 
$\MX^{RR} \in \Real{2n \times 2n}$, $\MX^{tt} \in \Real{n \times n}$,
and $\MX^{Rt} = (\MX^{tR})\tran \in \Real{2n \times n}$. 
%
\red{Problem~\eqref{eq:l2OnPosesPrimal} can be written as function of $\MZ$ and $\MX$:}
 \bea
\!\!\!\min_{\MZ,\MX} & \sumalledges \!\!
\oomegat\ij \|[\MX]^{Rt}_{ij}\!-\![\MX]^{Rt}_{ii}\!-\!\bart_{ij}\|_2 
\!+\! \frac{\oomegaR_{ij}}{\sqrt{2}} \|[\MX]^{RR}_{ij} \!-\!\barR_{ij}\|_F  \nonumber \\
\subject & \; \MX = \MZ\tran \MZ , \; [\MX]^{RR}_{ii}=\eye_2, \; 
\red{\det{[\MZ]_{i}}=+1}, \nonumber \\
& i=1,\dots,n 
\label{eq:l2OnPosesPrimalZX}
\eea
where we used the notation $[\MX]^{RR}_{ij}$ to identify the $2\times 2$ block entry 
of $\MX^{RR}$ at row $i$ and column $j$, and $[\MX]^{Rt}_{ij}$ to identify the $2 \times 1$ block entry 
of $\MX^{Rt}$, according to the partition shown in~\eqref{blockstruct}.
The optimization problem imposes the constraint that $\MX$ 
can be computed from $\MZ$, since $\MX = \MZ\tran \MZ$, while the constraints 
$[\MX]^{RR}_{ii} = \MR_i\tran \MR_i = \eye_2$ and $\red{\det{[\MZ]_{i}}=+1}$ enforce that the estimated 
rotation matrices are orthogonal.\footnote{By definition 
$\SOtwo \doteq \{\MR\in\Real{2\times 2}\!\!: \!\MR\tran \! \MR \!=\! \MR \!\MR\tran\!\!=\! \eye_2, \det(\MR)\!=\!+1\}$.}
In the following, 
we drop the determinant constraint for simplicity, and perform estimation over the \emph{Orthogonal group}
 rather than the Special Orthogonal group
\SOtwo. Empirical evidence from this and previous works, such as~\cite{Carlone15iros-duality3D}, shows that dropping the determinant constraint does not impact the relaxation.

Now we note that the constraint $\MX = \MZ\tran \MZ$ is equivalent to 
(i) $\MX \succeq 0$, and 
(ii) $\rank{\MX} = 2$, hence we can rewrite~\eqref{eq:l2OnPosesPrimalZX} in the sole variable $\MX$ as follows:
 \bea
\!\!\!\min_{\MX} & \sumalledges \!\!
\oomegat\ij \|[\MX]^{Rt}_{ij}\!-\![\MX]^{Rt}_{ii}\!-\!\bart_{ij}\|_2 
\!+\! \frac{\oomegaR_{ij}}{\sqrt{2}} \|[\MX]^{RR}_{ij} \!-\!\barR_{ij}\|_F  \nonumber \\
\subject & \;\; \MX \succeq 0, \; \rank{\MX} = 2, \; [\MX]^{RR}_{ii}=\eye_2, i=1,\dots,n 
\label{eq:l2OnPosesPrimalX}
\eea
Problem~\eqref{eq:l2OnPosesPrimalX} has a convex objective and the only nonconvex 
constraint is the rank constraint on $\MX$.
Therefore, to obtain a convex problem we  drop the rank constraint 
and we obtain the following convex relaxation.

\begin{proposition}[Convex Relaxation with $\ltwo$-norm]
\label{prop:l2OnPosesPrimal-relax}
The following semidefinite program is a convex relaxation of
 the nonconvex problem~\eqref{eq:l2OnPosesPrimal}: 
\begin{align}
\!\!\!\min_{\MX} & \sumalledges \!\!
\oomegat\ij \|[\MX]^{Rt}_{ij}\!-\![\MX]^{Rt}_{ii}\!-\!\bart_{ij}\|_2 
\!+\! \frac{\oomegaR_{ij}}{\sqrt{2}} \|[\MX]^{RR}_{ij} \!-\!\barR_{ij}\|_F  \nonumber \\
\subject & \;\; \MX \succeq 0, \;\;\;\; [\MX]^{RR}_{ii}=\eye_2,\quad i=1,\dots,n 
\label{eq:l2OnPosesPrimal-relax}
\end{align}
\end{proposition}

The fact that~\eqref{eq:l2OnPosesPrimal-relax} 
is a convex relaxation~\eqref{eq:l2OnPosesPrimal} trivially follows from 
the first part of this section: 
Problem~\eqref{eq:l2OnPosesPrimal-relax} is convex and its feasible set 
contains the one of Problem~\eqref{eq:l2OnPosesPrimalX}, 
which is simply a reformulation of Problem~\eqref{eq:l2OnPosesPrimal}.

\subsubsection{Relaxation of the $\lone$-norm formulation}
The convex relaxation of Problem~\eqref{eq:l1OnPosesPrimal} 
can be derived in full analogy with the one presented in the previous section, 
by introducing the matrix $\MX$ and noting that the terms $\MR_i\tran \MR_j$, 
$\MR_i\tran \vt_i$, and $\MR_i\tran \vt_j$ can be written using $\MX$.
Therefore, we obtain the following  

\begin{proposition}[Convex Relaxation with $\lone$-norm]
\label{prop:l1OnPosesPrimal-relax}
The following semidefinite program is a convex relaxation of
 the nonconvex problem~\eqref{eq:l1OnPosesPrimal}: 
\begin{align}
\!\!\!\min_{\MX} & \sumalledges \!\!
\oomegat\ij \|[\MX]^{Rt}_{ij}\!-\![\MX]^{Rt}_{ii}\!-\!\bart_{ij}\|_1 
\!+\! \frac{\oomegaR_{ij}}{2} \|[\MX]^{RR}_{ij} \!-\!\barR_{ij}\|_1  \nonumber \\
\subject & \;\; \MX \succeq 0, \;\;\;\; [\MX]^{RR}_{ii}=\eye_2,\quad i=1,\dots,n .
\label{eq:l1OnPosesPrimal-relax}
\end{align}
\end{proposition}

\subsubsection{Relaxation of Huber formulation}
This paragraph provides a relaxation of the robust formulation~\eqref{eq:huberOnPosesPrimal}, 
using the same derivation of the previous sections.

\begin{proposition}[Convex Relaxation with Huber loss]
\label{prop:huberOnPosesPrimal-relax}
The following semidefinite program is a convex relaxation of
 the nonconvex problem~\eqref{eq:huberOnPosesPrimal}: 
\bea
\!\!\!\min_{\MX} \!\!\!& \!\!\!\sumalledges \!\!
\emph{\huber}( \oomegat\ij \|[\MX]^{Rt}_{ij}\!-\![\MX]^{Rt}_{ii}\!-\!\bart_{ij}\|_2 ) 
\!+\!
\emph{\huber}( \oomegaR_{ij} \|[\MX]^{RR}_{ij} \!-\!\barR_{ij}\|_F )  \nonumber  \!\!\! \!\!\!\!\!\! \\
\subject \!\!\!& \;\; \MX \succeq 0, \;\;\;\; [\MX]^{RR}_{ii}=\eye_2,\quad i=1,\dots,n . \!\!\!\!\!\!\!\!\!
\label{eq:huberOnPosesPrimal-relax}
\eea
\end{proposition}
Again, the proof of the statement follows from the derivation 
we provided for the $\ltwo$-norm case. While the convexity of the 
costs in~\eqref{eq:l2OnPosesPrimal-relax} and~\eqref{eq:l1OnPosesPrimal-relax}
trivially follows from the convexity of the $\lone$ and $\ltwo$ norms, 
the convexity of the cost in~\eqref{eq:huberOnPosesPrimal-relax} is less straightforward.
To ascertain convexity of the cost in~\eqref{eq:huberOnPosesPrimal-relax} we
note that the first term in the cost has the form ``$\huber(\|\vxx\|_2)$'', and 
is the composition of (i) a convex function ($\|\vxx\|_2$) and (ii) a function ($\huber(\cdot)$) 
which is convex and nondecreasing when restricted to nonnegative arguments, see~\cite[p. 617]{Hartley00}. 
Properties (i) and (ii) guarantee that the resulting function $\huber(\|\vxx\|_2)$
is convex, see~\cite[p. 84]{Boyd04book}. 
An analogous argument holds for the second summand in~\eqref{eq:huberOnPosesPrimal-relax}.

\subsection{Rounding the Relaxed Solutions}  
\label{sec:rounding}

The solution of each of the convex relaxations~\eqref{eq:l2OnPosesPrimal-relax},~\eqref{eq:l1OnPosesPrimal-relax},
and~\eqref{eq:huberOnPosesPrimal-relax} is a matrix, that we call $\hatMX$. 
However, our goal is to estimate a set of poses. In this section we show how to 
retrieve the poses $(\MR_i,\vt_i),  i=1,\ldots,n$, from $\hatMX$.

\red{
\myParagraph{Rotation rounding}
The computation of the rotation estimates from $\hatMX$ proceeds along the same lines of~\cite{Wang13ima} 
and~\cite{Carlone16tro-duality2D}. 
Given the matrix $\hatMX$ we first compute a rank-$2$ approximation of $\hatMX$ via singular 
value decomposition. 
The resulting 
matrix $\hatMZ \in \Real{2 \times 3n}$  
is such that $\hatMZ\tran \hatMZ \approx \hatMX$ (the previous is an actual equality  
when $\hatMX$ has rank 2). Then, since the $\hatMZ$ has the structure described in~\eqref{eq:MZ}, 
we know that the first $n$ $2\times 2$ blocks of $\hatMZ$ contain our rotation estimates.
Therefore we project each block to \SOtwo, 
as prescribed in~\cite[Section 5]{Hartley13ijcv}, and obtain our rotation estimates $\hatMR_i$, $i=1,\ldots,n$.}


\red{
\myParagraph{Translation rounding}
The translation computation leverages the rotation estimates $\hatMR_i$, $i=1,\ldots,n$, 
computed in the previous section. Let us call $\hatMR \doteq [\hatMR_1 \; \ldots \; \hatMR_n]$ 
the matrix stacking all these estimates. 
Then, by inspection of the matrix $\MX$ in~\eqref{eq:MX} we realize that
 $\MR \MX^{Rt} = n [\vt_1 \; \ldots \; \vt_n]$.
Based on this observation, we build our translation estimate as:
\begin{align}
\label{eq:translations1}
\hatvt \doteq [\hatvt_1 \; \ldots \; \hatvt_n] = \frac{1}{n} \hatMR \hatMX^{Rt} .
\end{align}
}
\section{Decoupling Rotations and Translations: 2-stage Convex Relaxations} 
\label{sec:2stage}

In this section we discuss a slightly different approach to compute 
approximate solutions for the
 problems~\eqref{eq:l2OnPosesPrimal},~\eqref{eq:l1OnPosesPrimal},~\eqref{eq:huberOnPosesPrimal}.
These three problems share the fact that the corresponding cost functions
have two different terms: the first involving rotation and translations and the second involving only
rotations. Based on the empirical observation that the contribution of the first term to the 
rotation estimate is often negligible (see~\cite{Carlone14tro}), in this section 
we propose to first compute a rotation estimate by minimizing the second summand, 
and then compute the translation estimate by minimizing the first summand with given rotations.
While the decoupling of rotation and translation estimation might sound  
counterintuitive, Section~\ref{sec:8} shows that this \emph{2-stage approach} is indeed the most effective, 
leading to accurate estimate in presence of many outliers. 

\myParagraph{Stage 1: rotation estimation} 
The rotation subproblem in 
Eqs.~\eqref{eq:l2OnPosesPrimal},~\eqref{eq:l1OnPosesPrimal},~\eqref{eq:huberOnPosesPrimal}, 
assumes the following general form:
\beq
\!\!\! \min_{\MR_{i} \in \SOtwo} \sumalledges \!\!
 f_R( \MR_{i}\tran \!\MR_j \!-\!\barR_{ij} ) 
\label{eq:rotSubproblem}
\eeq
where, depending on the formulation 
(\eqref{eq:l2OnPosesPrimal},~\eqref{eq:l1OnPosesPrimal},~\eqref{eq:huberOnPosesPrimal}), 
the function $f_R(\cdot)$ denotes the $\ltwo$-norm, the $\lone$-norm, or the Huber loss, properly 
weighted by $\oomegaR_{ij}$. 
We now propose a convex relaxation which proceeds along the lines 
of Section~\ref{sec:relaxAndRound}.
First we define two matrices, $\MR \doteq [ \MR_1 \; \dots \; \MR_n] \in \Real{2\times2n}$,
 and
\beq
\label{eq:MXRR}
\MX^{RR} \doteq \MR\tran \MR
=
\left[\begin{array}{c c c} 
\MR_1\tran \MR_1 & \dots & \MR_1\tran \MR_n \\ 
 \vdots & \ddots & \vdots  \\ 
 \MR_n\tran  \MR_1 & \dots & \MR_n\tran \MR_n 
 \end{array} \right]  
 \in \Real{2n\times2n}
\eeq
Then, we repeat the same derivation of Section~\ref{sec:relaxAndRound} (but applied to the smaller matrix $\MX^{RR}$) 
and get the following convex relaxation of the rotation subproblem~\eqref{eq:rotSubproblem}:
\begin{align}
\!\!\!\min_{\MX} & \sumalledges \!\! f_R( [\MX]^{RR}_{ij} \!-\!\barR_{ij} )  \nonumber \\
\subject & \;\; \MX^{RR} \succeq 0, \;\;\;\; [\MX]^{RR}_{ii}=\eye_2,\quad i=1,\dots,n .
\label{eq:huberOnRot-relax}
\end{align}
Calling $\hatMX^{RR}$ the solution of the convex program~\eqref{eq:huberOnRot-relax},  
 we can recover the rounded rotation estimates, $\hatMR_i$, $i=1,\ldots,n$, from $\hatMX^{RR}$ as discussed in Section~\ref{sec:rounding}.
The relaxation~\eqref{eq:huberOnRot-relax} is similar to the one used by Wang and Singer in~\cite{Wang13ima}, 
except for the fact that we accommodate different robust cost functions.

\myParagraph{Stage 2: translation estimation} 
The translation subproblem  corresponds to the first summand in 
Eqs.~\eqref{eq:l2OnPosesPrimal},~\eqref{eq:l1OnPosesPrimal},~\eqref{eq:huberOnPosesPrimal},
and 
can be written in the following general form:
\beq
\!\!\! \min_{\vt_{i} \in \Real{2}}  \sumalledges \!\! f_t( \hatMR_{i}\tran\vt_j - \hatMR_{i}\tran \vt_i - \bart_{ij} ) 
\label{eq:tranSubproblem}
\eeq
where, depending on the formulation 
(\eqref{eq:l2OnPosesPrimal},~\eqref{eq:l1OnPosesPrimal},~\eqref{eq:huberOnPosesPrimal}), 
the function $f_t(\cdot)$ denotes the $\ltwo$-norm, the $\lone$-norm, or the Huber loss, properly 
weighted by $\oomegat_{ij}$. In~\eqref{eq:tranSubproblem}, we already 
substituted the rotation estimates computed in Stage 1, hence the translations are the only unknowns.

Problem~\eqref{eq:tranSubproblem} is already a convex program, since the problem is unconstrained 
and $f_t(\cdot)$ is a convex function in all the considered formulations.
Therefore, Stage 2 \red{simply consists of solving}~\eqref{eq:tranSubproblem} with an off-the-shelf
convex solver to get the translation estimates $\hatvt \doteq [\hatvt_1,\ldots, \hatvt_n]$. 
%

\section{A Posteriori Performance Guarantees} 
\label{sec:guarantees}

This section provides \emph{a posteriori} 
checks that can ascertain the quality of the relaxed solution after solving 
the convex relaxation.
We give the following result, applied to~\eqref{eq:l2OnPosesPrimal}, while the 
very same statement holds for Problems~\eqref{eq:l1OnPosesPrimal},~\eqref{eq:huberOnPosesPrimal}. 
\begin{proposition}[Tightness in robust PGO formulations]
\label{prop1}
Let $f(\MR,\vt)$ be the cost function of the robust PGO formulation~\eqref{eq:l2OnPosesPrimal}, 
with $\MR \doteq [\MR_1,\ldots,\MR_n]$ and $\vt \doteq [\vt_1,\ldots,\vt_n]$, and 
call $f^\star$ the corresponding optimal cost.
With slight abuse of notation, call $f(\MX)$ 
the cost function of the convex relaxation~\eqref{eq:l2OnPosesPrimal-relax}, 
and denote with $\hatMX$ the optimal solution of~\eqref{eq:l2OnPosesPrimal-relax}
and with $(\hatMR,\hatvt)$ the corresponding rounded estimate.
Then, it holds:
\begin{align}
\label{eq:poseSubopt}
f(\hatMR,\hatvt) - f^\star \leq f(\hatMR,\hatvt) - f(\hatMX)
\end{align}
i.e., we can compute a bound on the suboptimality gap $f(\hatMR,\hatvt) - f^\star$ 
by using the optimal cost of the relaxation $f(\hatMX)$.
Moreover, if the relaxed solution $\hatMX$ has rank 2 
and its rank-2 decomposition $\hatMZ$ is such that the first n $2\times2$ blocks of $\hatMZ$ are in \SOtwo,
then $f(\hatMR,\hatvt) = f^\star$ and the rounded solution $(\hatMR,\hatvt)$ 
is optimal for the original problem~\eqref{eq:l2OnPosesPrimal}.
\end{proposition}

\begin{proof}
Since~\eqref{eq:l2OnPosesPrimal-relax} is a relaxation of Problem~\eqref{eq:l2OnPosesPrimal}, 
it follows that $f(\hatMX) \leq f^\star$, which implies the inequality~\eqref{eq:poseSubopt}.
 Moreover, if $\hatMX$ has rank 2, then $f(\hatMX) = f(\hatMZ\tran \hatMZ)$, where 
 $\hatMZ$ is the rank-2 decomposition of $\hatMX$. 
 If the first n $2\times2$ blocks of $\hatMZ$ are already in \SOtwo, then the rounding 
 does not alter $\hatMZ$, i.e., $\hatMZ = [\hatMR \; \hatvt]$.
 Therefore, it holds that (i) $f(\hatMX) = f(\hatMR,\hatvt) \leq f^\star$. Since $(\hatMR,\hatvt)$ 
 is feasible for~\eqref{eq:l2OnPosesPrimal}, by optimality of $f^\star$ it follows that
 (ii) $f^\star \leq f(\hatMR,\hatvt)$. Combining the inequalities (i) and (ii), 
 it follows that $f(\hatMR,\hatvt) = f^\star$, proving the optimality of $(\hatMR,\hatvt)$.
\end{proof}

Proposition~\ref{prop1} provides computational tools to quantify the suboptimality 
of the rounded solution $(\hatMR,\hatvt)$. Moreover, it gives an a posteriori condition under 
which the relaxation is tight. Tightness is attained under 2 conditions. 
The first condition is that
the rank of $\hatMX$ is 2: this guarantees that we did not lose anything 
when relaxing the rank constraint in~\eqref{eq:l2OnPosesPrimalX}.
The second condition is that $\hatMZ$ (the rank 2 decomposition of $\hatMX$)
contains rotation matrices: this guarantees that we did not lose anything 
by dropping the determinant constraint in~\eqref{eq:l2OnPosesPrimalX}.
As mentioned in Section~\ref{sec:relaxAndRound}, indeed empirical evidence suggests  
that the determinant constraint 
does not impact the results, hence we are mainly interested in the rank of $\hatMX$.

An analogous result applies to the rotation estimation of the 2-stage approaches, 
as discussed below (the proof is identical to the one of Proposition~\ref{prop1}).

\begin{proposition}[Tightness in 2-stage formulations]
\label{prop2}
Let $f_R(\MR)$ be the cost function of the rotation subproblem in~\eqref{eq:rotSubproblem}, 
with $\MR \doteq [\MR_1,\ldots,\MR_n]$, and 
call $f_R^\star$ the corresponding optimal cost.
With slight abuse of notation, call $f_R(\MX^{RR})$ 
the cost function of the convex relaxation~\eqref{eq:huberOnRot-relax}, 
and denote with $\hatMX^{RR}$ the optimal solution of~\eqref{eq:huberOnRot-relax}
and with $\hatMR$ the corresponding rounded estimate.
Then, it holds that
\begin{align}
\label{eq:rotSubopt}
f_R(\hatMR) - f_R^\star \leq f_R(\hatMR) - f_R(\hatMX^{RR})
\end{align}
i.e., we can compute a bound on the suboptimality gap $f_R(\hatMR) - f_R^\star$ 
by using the optimal cost of the relaxation $f_R(\hatMX^{RR})$.
Moreover, if the relaxed solution $\hatMX^{RR}$ has rank 2 
and its rank-2 decomposition $\hatMZ$ is such that the first n $2\times2$ blocks of $\hatMZ$ are in \SOtwo,
then $f_R(\hatMR) = f_R^\star$ and the rounded solution $\hatMR$ 
is optimal for the rotation subproblem in~\eqref{eq:rotSubproblem}.
\end{proposition}

Note that we only discuss the tightness in the relaxation of the 
rotation estimation problem (Stage 1), while translation estimation (Stage 2) 
is already convex.



\section{Numerical Evaluation}
\label{sec:8}

This section presents extensive numerical simulations and 
provides empirical evidence that 
(i) the {\itshape 1-stage} convex relaxations of Section~\ref{sec:relaxAndRound} are not tight, 
while the relaxations of the rotation subproblems in Section~\ref{sec:2stage} are indeed tight 
in the tested scenarios; 
(ii) the {\itshape 2-stage} approaches are robust with respect to increasing probability of outlying measurements
 and ensure excellent performance in practice; 
(iii) the {\itshape 2-stage} approaches allows a reliable detection of the outliers;
\red{(iv) the {\itshape 2-stage} approaches outperform state-of-the-art techniques based on local optimization}, 
such as \emph{Dynamic Covariance Scaling}~\cite{Agarwal13icra}.



\subsection{\red{Summary of the Proposed Approaches}}

\red{In the experimental evaluation we compare the 6 proposed approaches for solving the PGO with outliers:}
\begin{enumerate}
\item $\ell_1$ {\itshape on poses}: solution of Problem~\eqref{eq:l1OnPosesPrimal-relax};
\item $\ell_1$ {\itshape 2-stage}: solution of Problems~\eqref{eq:huberOnRot-relax} and \eqref{eq:tranSubproblem} 
with $f_R(\cdot)=\|\cdot\|_1$ and $f_t(\cdot)= \|\cdot\|_1$;
\item $\ell_2$ {\itshape on poses}: solution of Problem~\eqref{eq:l2OnPosesPrimal-relax};
\item $\ell_2$ {\itshape 2-stage}: solution of Problems~\eqref{eq:huberOnRot-relax} and \eqref{eq:tranSubproblem} with $f_R(\cdot)= \|\cdot\|_F$ and $f_t(\cdot)= \|\cdot\|_2$;
\item {\itshape Huber on poses}: solution of Problem~\eqref{eq:huberOnPosesPrimal-relax};
\item {\itshape Huber 2-stage}:  solution of Problems~\eqref{eq:huberOnRot-relax} and \eqref{eq:tranSubproblem} with 
$f_R(\cdot) = \huber(\|\cdot\|_F)$ and $f_t(\cdot) = \huber(\|\cdot\|_2)$. 
\end{enumerate}

All problems above are convex and
can be solved globally by off-the-shelf SDP solvers. We use \cvx \cite{CVXwebsite} 
as convex solver. 
After solving the convex relaxation, we apply the \emph{rounding procedure}, in order to obtain a feasible solution of the original problems.
We compute statistics over $30$ runs.

\subsection{Monte Carlo Analysis and Simulation Setup}
\label{sec:setup}

We performed a Monte Carlo analysis on synthetic datasets consisting of randomly generated pose graphs. 
Ground truth positions of the $n$ nodes are drawn from a uniform distribution 
in a square;
ground truth
 orientations are drawn from a uniform distribution over $\mpipi$. Connections among nodes are generated according to (i) the {\itshape Erd\H os-R\'enyi} random graph model, where each edge is included in the graph with probability $p$ independently from every other edge 
 (in our tests we set $p=0.5$); and (ii) the {\itshape Geometric} random graph model, where all nodes closer than a specified distance/radius (in our tests, we used as default value $\frac{\Delta}{4}$ where $\Delta$ is the size of the environment) are connected by an edge.  
 \red{Examples of random graphs are provided in Fig.~\ref{fig:graphs}(a-b).}
Since we 
are interested in connected graphs, random graphs that contain disconnected components 
 are discarded.
Once the graph is generated, the relative measurements 
$(\cht\ij, \chR_{ij})$ associated to each edge $(i,j) \in \calE$ are generated as:
\begin{align}
\label{eq:modelOutTran}
\cht\ij &=(1-\delta_{ij})\bart\ij+\delta_{ij}\tilt\ij, \\
\chR\ij &=(1-\delta_{ij})\barR\ij+\delta_{ij}\tilR\ij \nonumber
\end{align}
where $\delta_{ij}$ are i.i.d. {\itshape Bernoulli} random variables 
with parameter $p^{out}$, and are 
such that with probability $p^{out}$, 
$\delta_{ij}=1$ and the measurement $(\cht\ij, \chR_{ij})$ is assigned an outlier measurement $(\tilR\ij,\tilt\ij)$, 
while if $\delta_{ij}=0$ the measurement $(\cht\ij, \chR_{ij})$ is assigned 
an  inlier but noisy measurement $(\barR\ij,\bart\ij)$. 
Inliers and outliers models are as follows:
 the inliers $(\bart\ij, \barR\ij)$ are {\itshape noisy measurements} generated according to model 
\eqref{eq:noisyModel} with $\noiset\ij \sim \gaussian(\mathbf{0}_2, \sigma^2_{T})$ 
and $\noiseR_{ij} = \MR(\noiseang\ij)$, $\epsilon^{R}_{ij} \sim \gaussian(0, \sigma^2_{R})$,
where $\MR(\noiseang\ij)$ is the planar rotation matrix of angle $\noiseang\ij$.
For our tests we set $\sigma_R = 0.01$ and $\sigma_T = 0.1$, which are good proxies for the 
noise levels found in practice.
The outliers $(\tilt\ij, \tilR\ij)$ are completely wrong measurements, and are obtained from the
model \eqref{eq:noisyModel} by adding large uniformly distributed noise: $\noiset\ij \sim \uniform(-\frac{\Delta}{4}, 
\frac{\Delta}{4})$
(where $\Delta$ is the size of the environment), and $\noiseR_{ij} = \MR(\noiseang\ij)$, $\epsilon^{R}_{ij} \sim \uniform(-\pi,+\pi)$. 
Since the general purpose solver in \cvx does not scale to large instances, we focus on relatively 
small graphs, with $n=20$ and $n=50$.


  \begin{figure}[h]
\begin{minipage}{\columnwidth}
\begin{tabular}{cc}
\hspace{-5 mm}
\begin{minipage}{0.5\columnwidth}%
\includegraphics[scale=0.37, trim=0cm 0cm 0cm 0.5cm,clip]{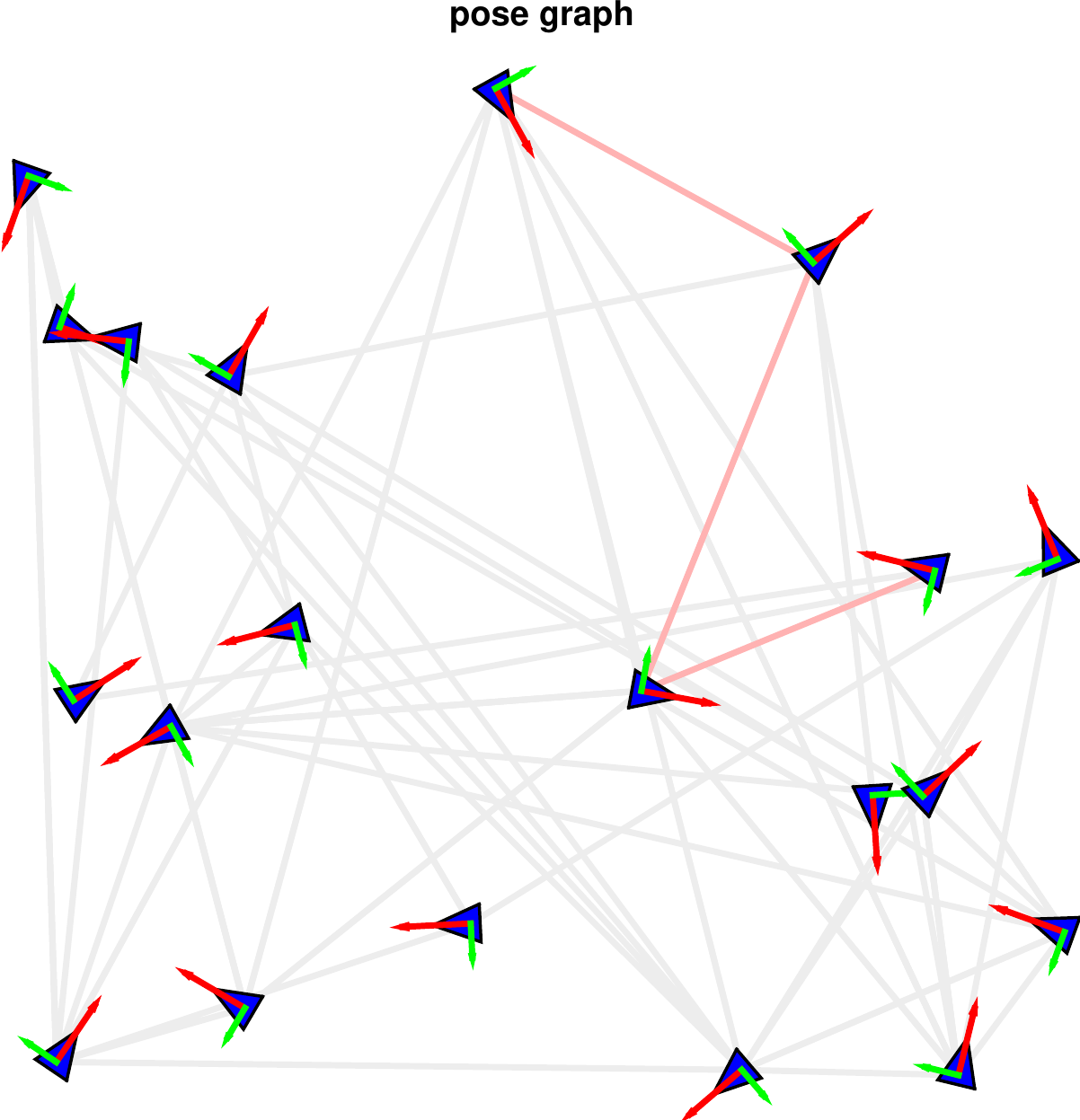}  \\ 
\centering (a) 
\end{minipage}%
&
\hspace{ -3mm}
\begin{minipage}{0.5\columnwidth}%
\centering
\includegraphics[scale=0.37]{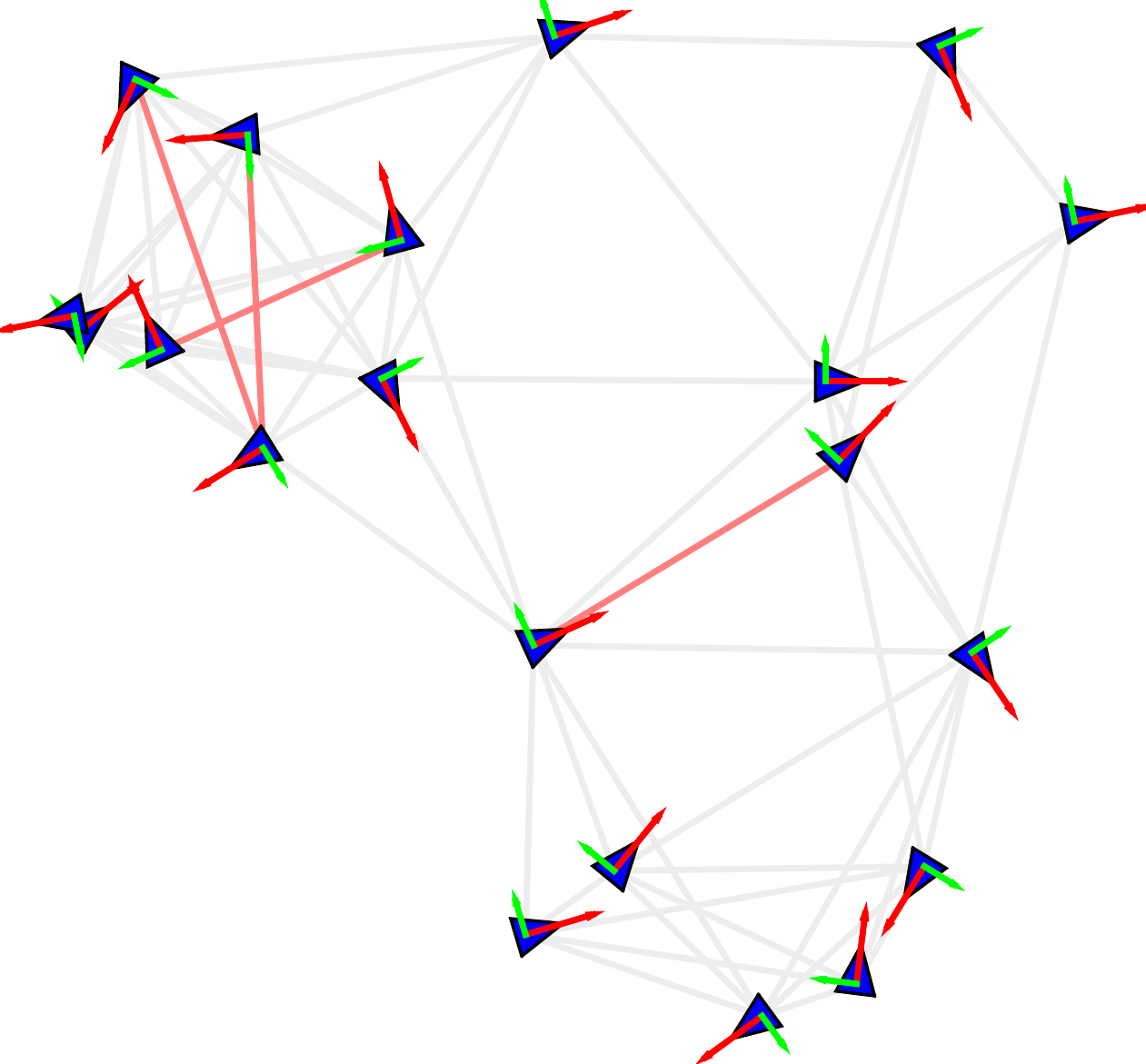}  \\ 
\centering (b) 
\end{minipage}%
\end{tabular} 
\end{minipage}%
\\
\vspace{1 mm}
\begin{tabular}{cc}
\hspace{-5 mm}
\begin{minipage}{1\columnwidth}%
\centering
\includegraphics[scale=0.5, angle =90]{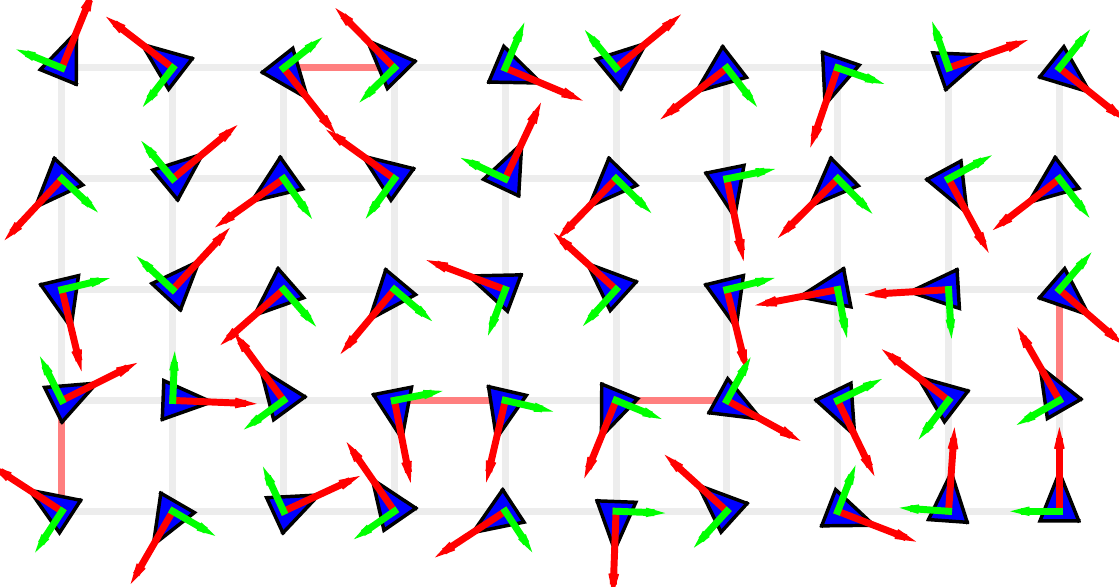}  \\ 
\centering (c) 
\end{minipage}%
\end{tabular} 
\red{\caption{Examples of simulated pose graphs: (a) Erd\H os-R\'enyi random graph; (b) Geometric random graph; 
(c) Grid graph.} 
\vspace{-5 mm}
\label{fig:graphs} }
\end{figure}


\subsection{Tightness of convex relaxations}
In this section we evaluate the tightness of the convex relaxations, i.e., their 
capability of producing accurate solutions for the original (nonconvex) problem.
In order to check if the convex relaxations in {\itshape 1-stage} approaches are tight, we check
 the rank of the solution matrices $\hatMX$ of problems \eqref{eq:l2OnPosesPrimal-relax}, \eqref{eq:l1OnPosesPrimal-relax} and \eqref{eq:huberOnPosesPrimal-relax} (see Proposition~\ref{prop1}).
 In all instances we tested, 
the rank of matrix $\hatMX$ is very high, even in absence of outliers ($p^{out}=0$). 
The solution matrices of problem \eqref{eq:l1OnPosesPrimal-relax} and \eqref{eq:huberOnPosesPrimal-relax} 
have rank equal to $n$, while the rank of the solution matrix $\hatMX$ of \eqref{eq:l2OnPosesPrimal-relax} 
is around $10$ and $25$ for instances having respectively $20$ and $50$ nodes. 
This empirical evidence shows that the SDP relaxations \eqref{eq:l2OnPosesPrimal-relax}, \eqref{eq:l1OnPosesPrimal-relax} and \eqref{eq:huberOnPosesPrimal-relax} are not tight in presence of noise and outlying measurements. 
 


On the other hand, in {\itshape 2-stage approaches} the relaxations of the rotation subproblem 
(i.e. problem~\eqref{eq:huberOnRot-relax}) are practically always tight. We computed the {\itshape stable rank} of 
the matrix $\hatMX^{RR}$, i.e. the squared ratio between the Frobenius norm and the spectral norm.
The stable rank is a real number rather than an integer, and
 provides a more detailed picture of the rank of the matrix, without requiring to 
commit to a given numerical tolerance for the rank computation.
This allows understanding 
how far is $\hatMX^{RR}$ from being a rank-2 matrix.
Fig.~\ref{fig:rankGRR}(a) shows that the stable rank of $\hatMX^{RR}$ is very close to 2 for all considered values of $p^{out}$ and for $n=20$. 
For the 2-stage approaches, this confirms the tightness of our convex relaxations.
Interestingly, also if the 1-stage approaches, even if the matrix $\hatMX$ has large rank, 
the submatrix $\hatMX^{RR}$ has rank close to 2. This would suggest that the 
presence of the translations in the 1-stage formulations breaks the tightness. 
These results are further confirmed by Fig.~\ref{fig:rankGRR}(b), which shows results for 
larger pose 
graphs with $n=50$ poses.

\begin{figure}[h]
\begin{minipage}{\columnwidth}
\begin{tabular}{cc}
\hspace{-5 mm}
\begin{minipage}{0.5\columnwidth}%
\includegraphics[scale=0.25, trim = 0mm 0mm 0mm 10mm, clip]{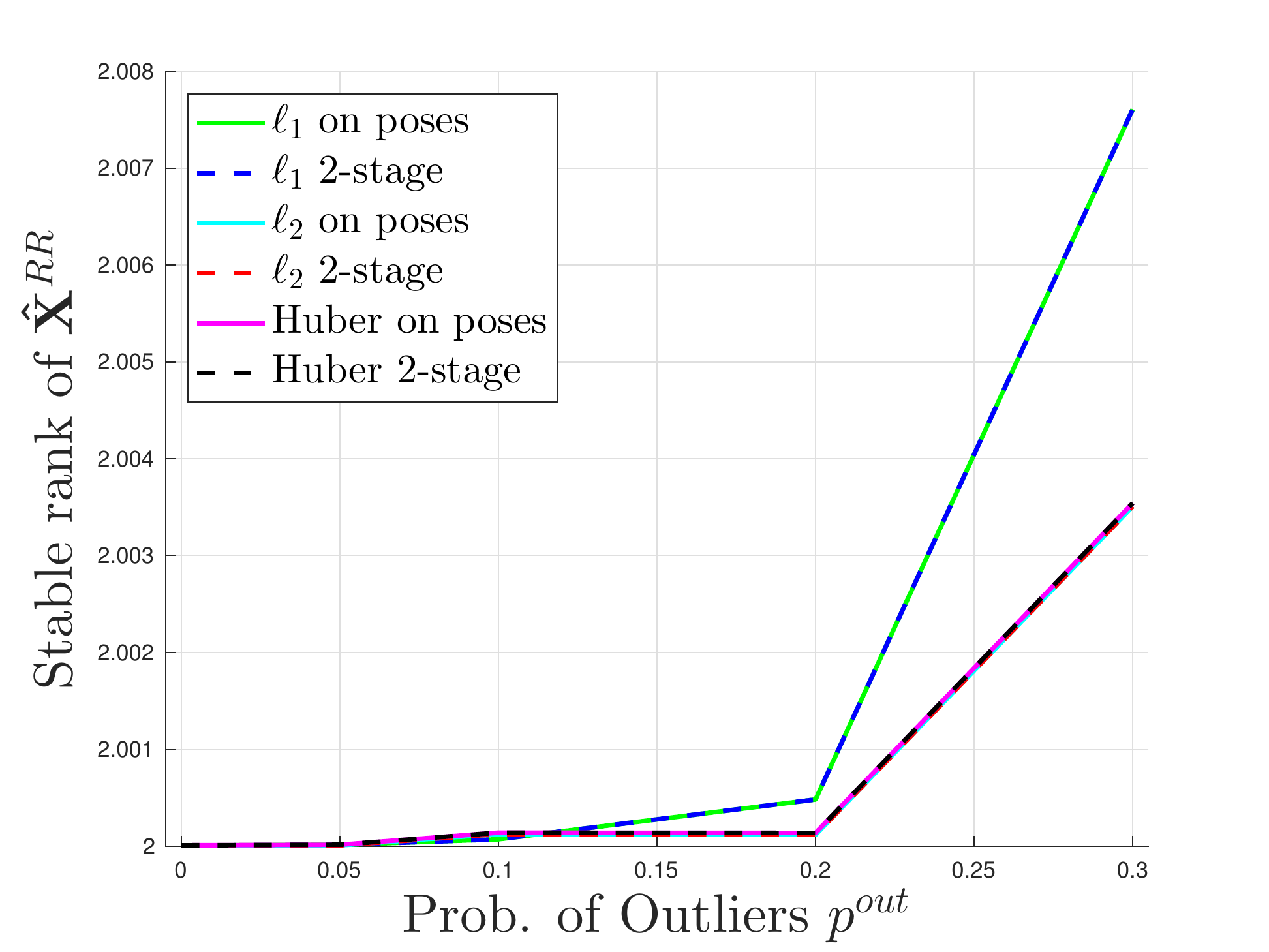}  \\ 
\centering (a) $n=20$ 
\end{minipage}
&
\hspace{-3 mm}
\begin{minipage}{0.5\columnwidth}%
\includegraphics[scale=0.25, trim = 0mm 0mm 0mm 10mm, clip]{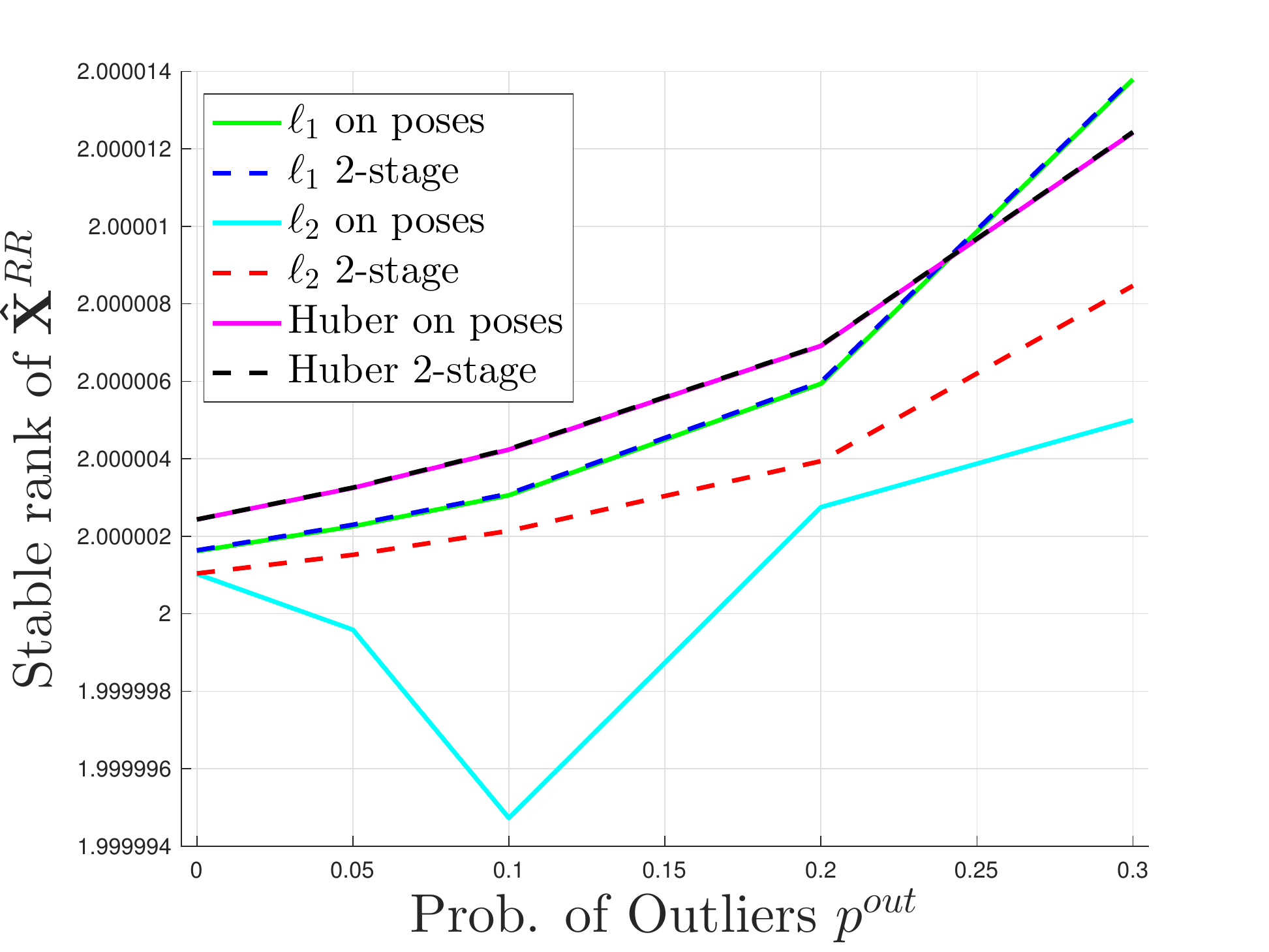}  \\ 
\centering (b) $n=50$ 
\end{minipage}%
\end{tabular}
\end{minipage}
\vspace{-2mm}
\caption{Stable rank of the matrix $\hatMX^{RR}$ computed by the different approaches and for increasing probability of outliers. 
\label{fig:rankGRR} }
\vspace{-5mm}
\end{figure}


\subsection{Robust \PGO}
\label{sec:robustPGO}

In this section we evaluate, 
for increasing levels of outliers $p^{out}$,
 the quality of the pose estimates produced by the proposed approaches.
The estimation quality is quantified by the mean rotation and mean translation error, computed with 
  respect to the ground truth. Fig.~\ref{fig:aveDistFromGT_small} shows the mean errors, averaged on $30$ runs.
  We note that there is a huge divide between the two set of approaches. 
  On the one hand, the 1-stage approaches show no robustness towards the presence of outliers, since mean translation and rotation errors quickly increase with $p^{out}$. 
On the other hand, 
  the estimates provided by the 2-stage approaches have small errors and are remarkably insensitive to the increase of the probability of outliers (y-axis is in log scale). 
  In particular, {\itshape Huber 2-stage} ensures top performance, followed by $\ell_2$ {\itshape 2-stage} and $\ell_1$ {\itshape 2-stage}.
  
\newcommand{\myscaleo}{0.26}
\newcommand{\myhspace}{-3.5mm}
\newcommand{\myhspaceo}{-3mm}

\begin{figure}[h]
\begin{minipage}{\columnwidth}
\hspace{\myhspace}
\begin{tabular}{cc}
\hspace{-5 mm}
\begin{minipage}{0.5\columnwidth}%
\includegraphics[scale=\myscaleo]{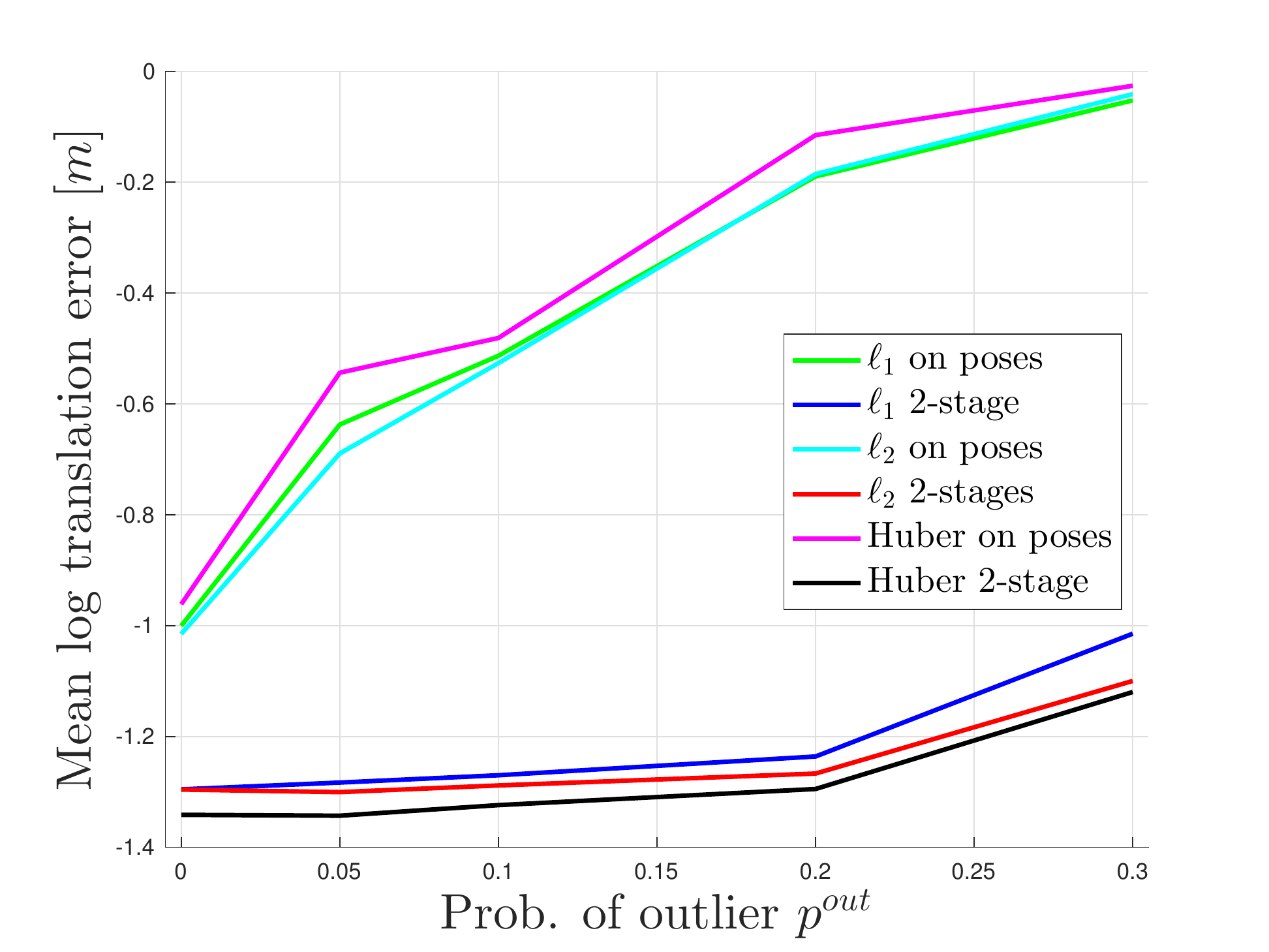}  \\ 
\end{minipage}%
&
\hspace{\myhspaceo}
\begin{minipage}{0.5\columnwidth}%
\centering
\vspace{-3.5mm}
\includegraphics[scale=\myscaleo,trim={0.5cm 0 0 0},clip]{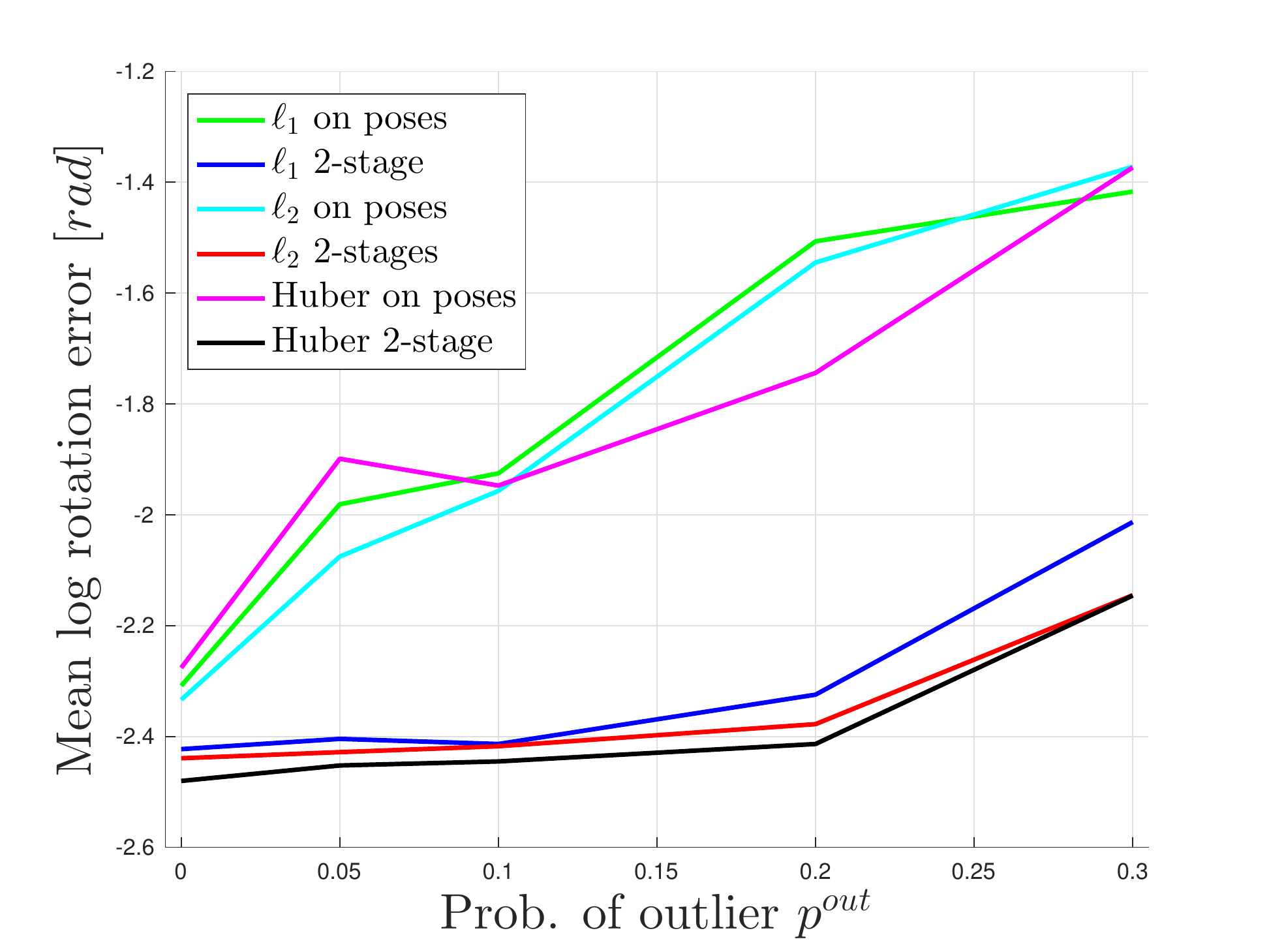}  \\ 
\end{minipage}%
\end{tabular} 
\vspace{-0.5cm} \\
\centering
(a) {\footnotesize  {\itshape Erd\H os-R\'enyi}, $n=20$.}
\end{minipage}%
\\
\begin{minipage}{\columnwidth}
\hspace{\myhspace}
\begin{tabular}{cc}
\hspace{-5 mm}
\begin{minipage}{0.5\columnwidth}%
\includegraphics[scale=\myscaleo]{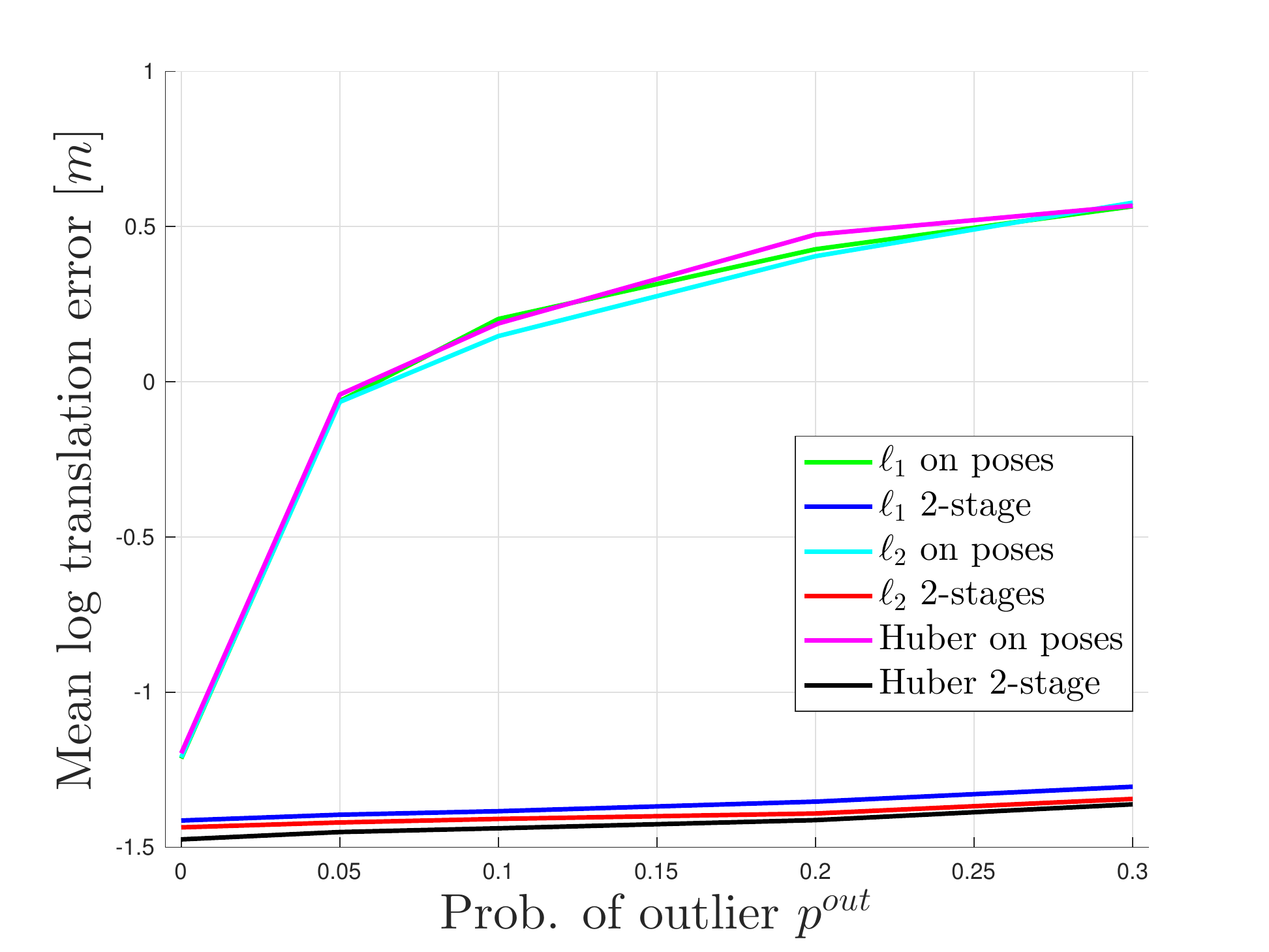}  \\ 
\end{minipage}%
&
\hspace{\myhspaceo}
\begin{minipage}{0.5\columnwidth}%
\centering
\vspace{-3.5mm}
\includegraphics[scale=\myscaleo,trim={0.5cm 0 0 0},clip]{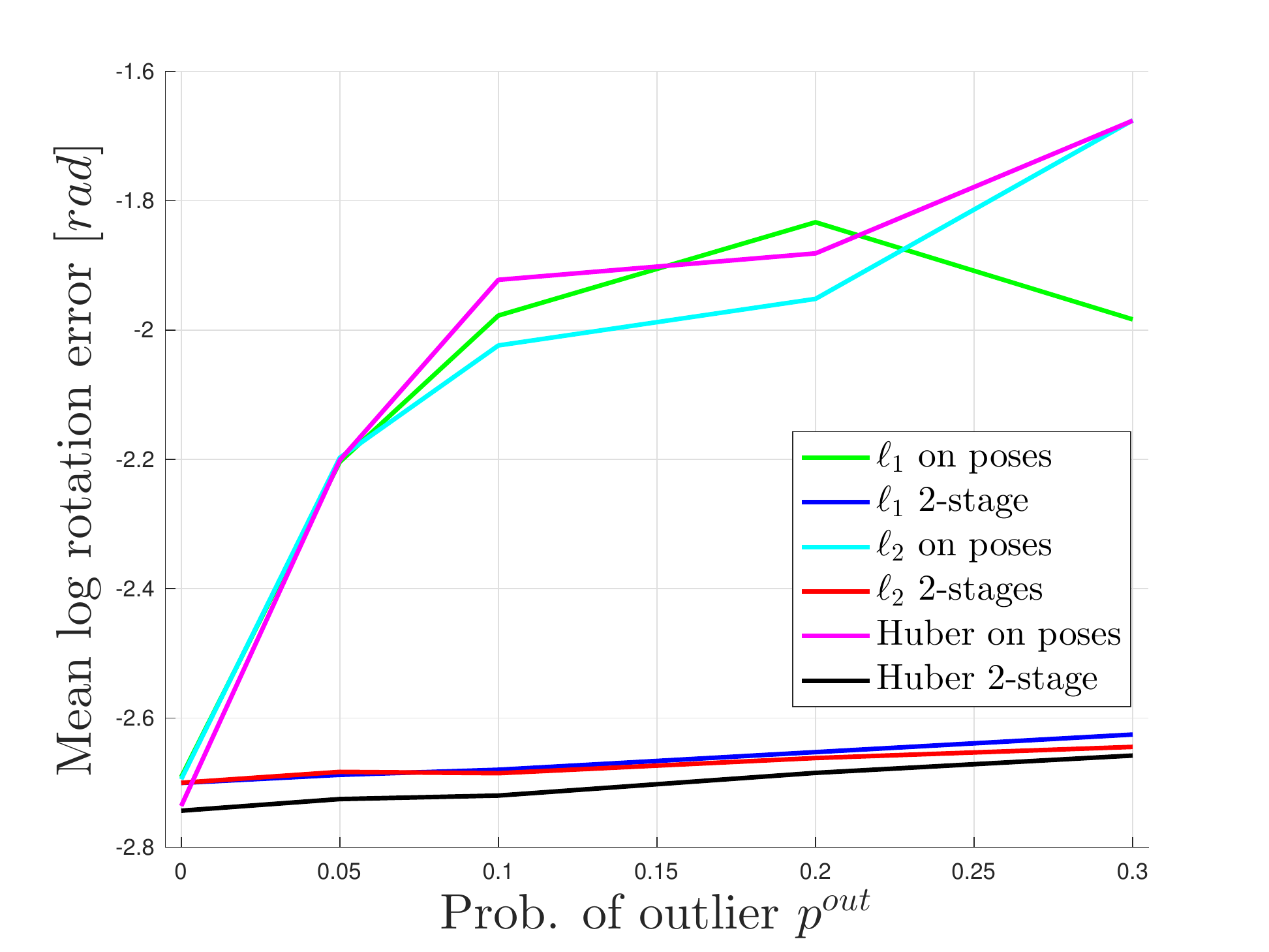}  \\ 
\end{minipage}%
\end{tabular}
\vspace{-0.5cm} \\
\centering
(b) {\footnotesize  {\itshape Erd\H os-R\'enyi}, $n=50$.} 
\end{minipage}%
\\
\begin{minipage}{\columnwidth}
\hspace{\myhspace}
\begin{tabular}{cc}
\hspace{-5 mm}
\begin{minipage}{0.5\columnwidth}%
\includegraphics[scale=\myscaleo]{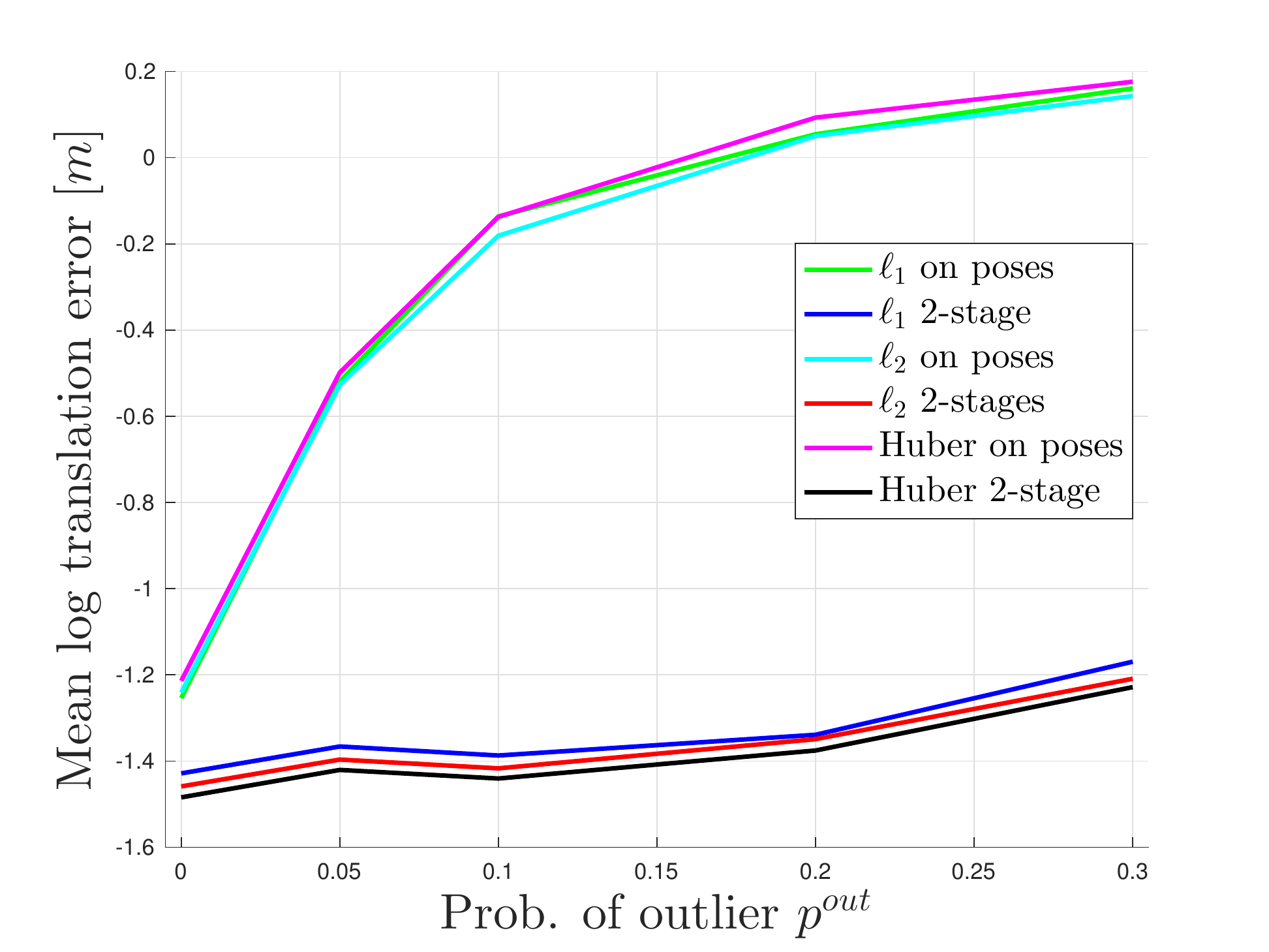}  \\ 
\end{minipage}%
&
\hspace{\myhspaceo}
\begin{minipage}{0.5\columnwidth}%
\centering
\vspace{-3.5mm}
\includegraphics[scale=\myscaleo,trim={0.5cm 0 0 0},clip]{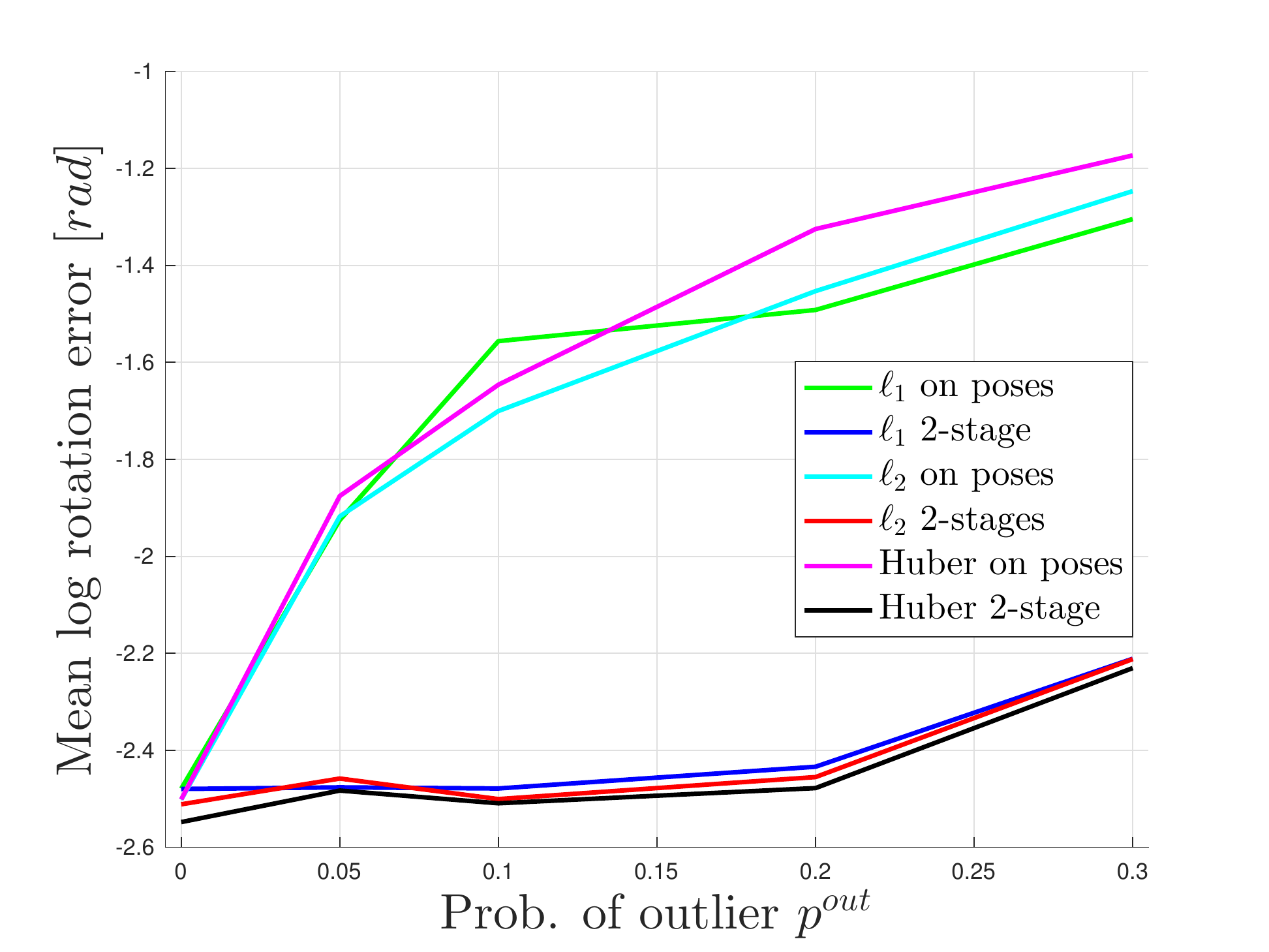}  \\ 
\end{minipage}%
\end{tabular} 
\vspace{-0.5cm} \\
\centering
(c) {\footnotesize {\itshape Geometric}, $n=20$.}
\end{minipage}%
\\
\begin{minipage}{\columnwidth}
\hspace{\myhspace}
\begin{tabular}{cc}
\hspace{-5 mm}
\begin{minipage}{0.5\columnwidth}%
\includegraphics[scale=\myscaleo]{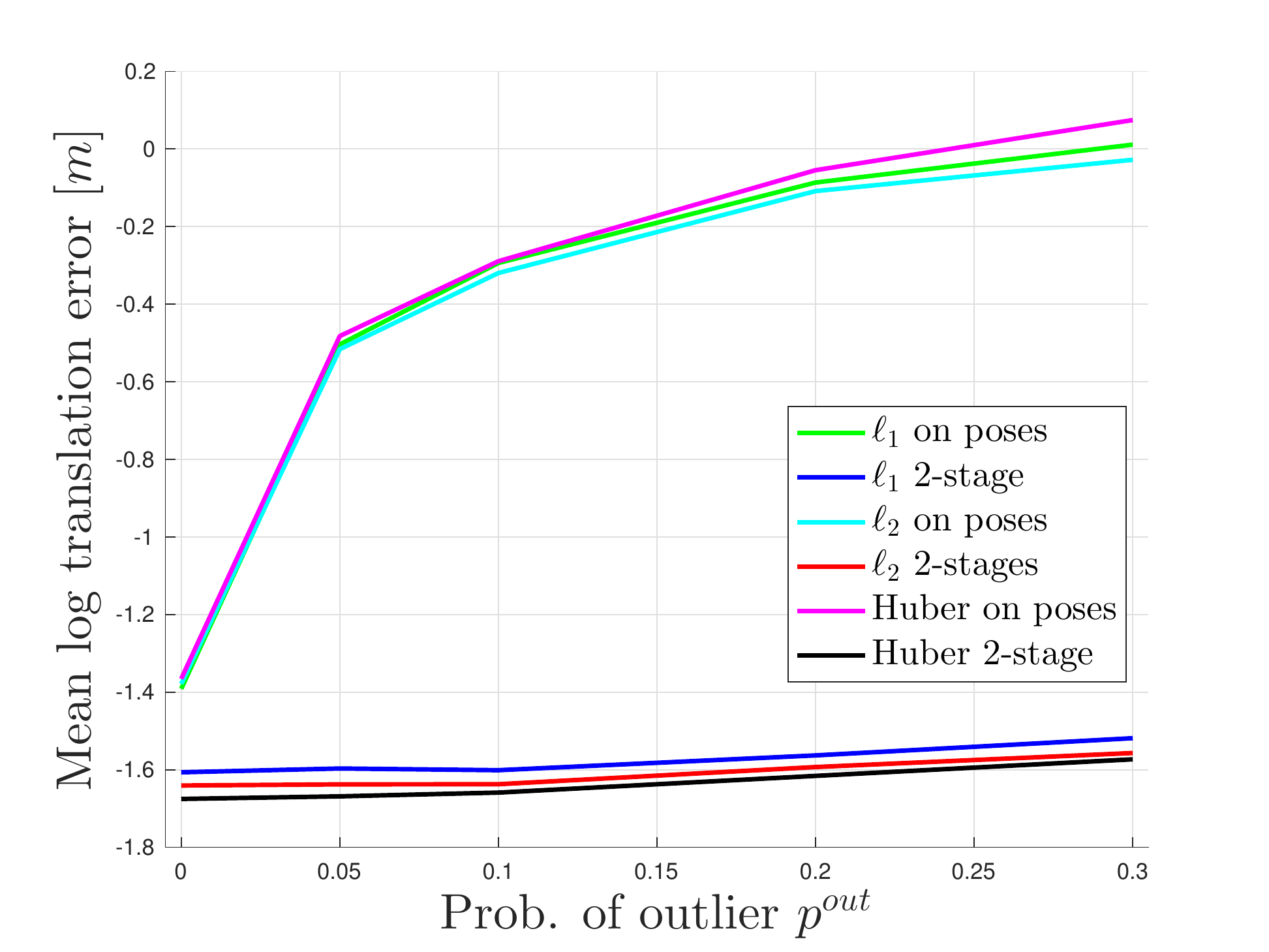}  \\ 
\end{minipage}%
&
\hspace{\myhspaceo}
\begin{minipage}{0.5\columnwidth}%
\centering
\vspace{-3.5mm}
\includegraphics[scale=\myscaleo,trim={0.5cm 0 0 0},clip]{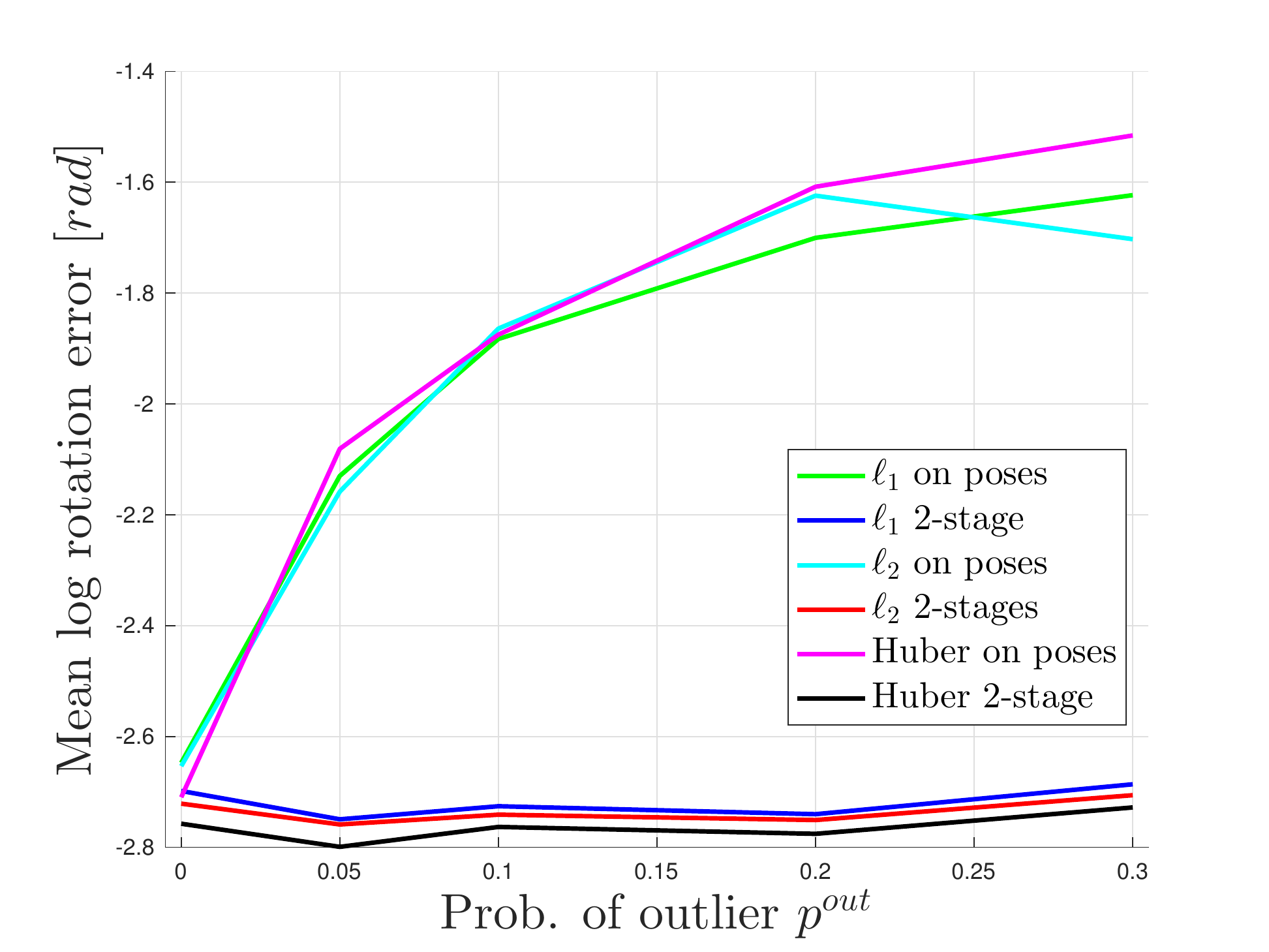}  \\ 
\end{minipage}%
\end{tabular} 
\vspace{-0.5cm} \\
\centering
(d) {\footnotesize {\itshape Geometric}, $n=50$.}
\end{minipage}%
\caption{Estimation errors for the 6 approaches proposed in this paper.
Left column: average translation error. Right column: average rotation error. 
 {\itshape Erd\H os-R\'enyi} random graphs with (a) $n=20$ and (b) $n=50$ nodes; {\itshape Geometric} random graphs with (c) $n=20$ and (d) $n=50$ nodes.
\label{fig:aveDistFromGT_small}}
\vspace{-4mm}
\end{figure}



\subsection{Sensor Failure Detection via Convex Relaxations}
\label{sec:OutIdentifResults}

In this section we discuss how to use the proposed techniques for outlier detection.
Besides being useful to improve the quality of the pose estimates, 
the capability of identifying outliers is an important tool for 
sensor failure detection.  

Here we show that since our techniques can produce an accurate 
estimate even in presence of a large amount of outliers, we can simply 
detect outliers by evaluating the residual error of each measurement.
Given an estimate $(\MR,\vt)$, with
 $\MR = [\MR_1,\ldots,\MR_n]$ and $\vt = [\vt_1,\ldots,\vt_n]$,
 and for each measurement $(\chR\ij,\cht\ij)$, we define the 
translation residual $\res\ij^t$ and the rotation residual $\res\ij^R$ as follows:
\begin{align}
\label{eq:residuals}
\res\ij^t(\MR,\vt) &= \| \MR_i \tran(\vt_j - \vt_i) -  \cht\ij \|_2 \\
\res\ij^R(\MR) &= \| \MR_i \tran \MR_j - \chR\ij \|_F .
\end{align}
A simple outlier detection technique would first solve one of the proposed convex 
relaxations to obtain an estimate of the pose graph configuration $(\hatMR,\hatvt)$.
Then it would use such estimate to compute the residual errors 
$\res\ij^t(\hatMR,\hatvt)$ and $\res\ij^R(\hatMR)$ for all $(i,j)\in\calE$.
Finally, it would classify as outliers the measurements for which either 
$\res\ij^t(\hatMR,\hatvt) \geq \eta^t$ or $\res\ij^R(\hatMR) \geq \eta^R$, where 
$\eta^t$ and $\eta^R$ are given thresholds on the maximum admissible error.


Fig.~\ref{fig:PRcurves} reports
 the precision/recall curves of the classification based on the solutions yielded by the 6 approaches 
 discussed in this paper. \red{The curves are obtained by varying the thresholds 
 $\eta^t$ and $\eta^R$ in the ranges $[0,\: 50]\: \text{m}$ and $[0,\: \pi]$,} respectively, and evaluating the precision/recall, defined as:
\begin{align}
\text{precision} = \frac{{\check n}_{out}}{({\check n}_{out}+{\bar n}_{out})}, \quad \quad  \quad
\text{recall} = \frac{{\check n}_{out}}{n_{out}}
\end{align}
where $n_{out}$ is the (true) number of outliers, ${\check n}_{out}$ is the number of correctly identified outliers 
 (true positives), and ${\bar n}_{out}$ is the number of inliers wrongly classified as outliers (false positives). 
As expected, the 2-stage techniques, which we have already seen to 
ensure accurate estimation independently on the amount of outliers 
lead to high precision/recall and are fairly insensitive to the choice of the thresholds $\eta^t$ and $\eta^R$.
\red{The 1-stage techniques, instead, lead to poorer performance.}

\vspace{-0.3cm}
\begin{figure}[h]
\begin{minipage}{\columnwidth}
\begin{tabular}{cc}
\hspace{-5 mm}
\begin{minipage}{0.5\columnwidth}%
\includegraphics[scale=0.235]{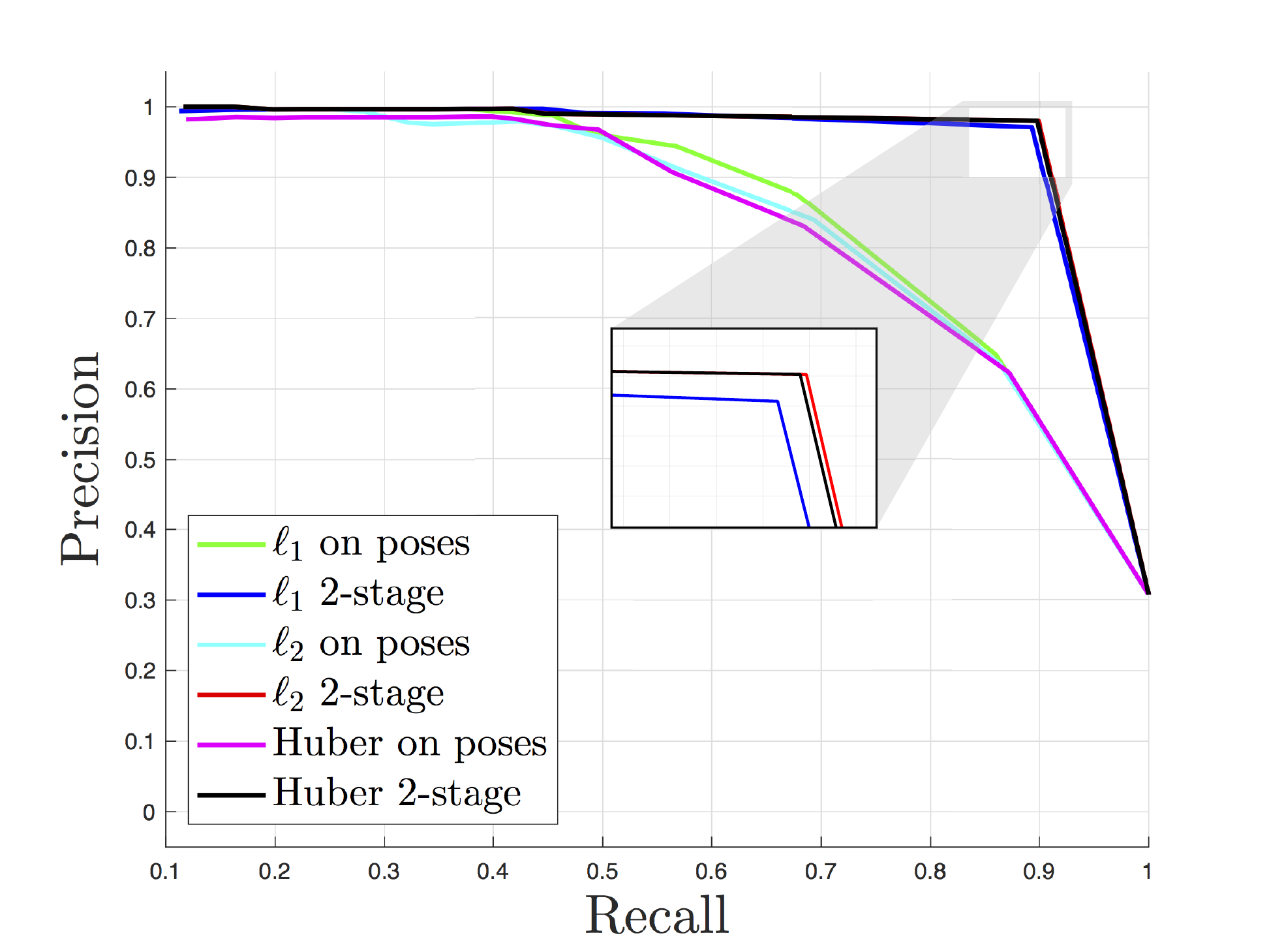}  \\ 
\centering (a) {\footnotesize $n=20$.}
\end{minipage}
&
\hspace{-5 mm}
\begin{minipage}{0.5\columnwidth}%
\includegraphics[scale=0.235]{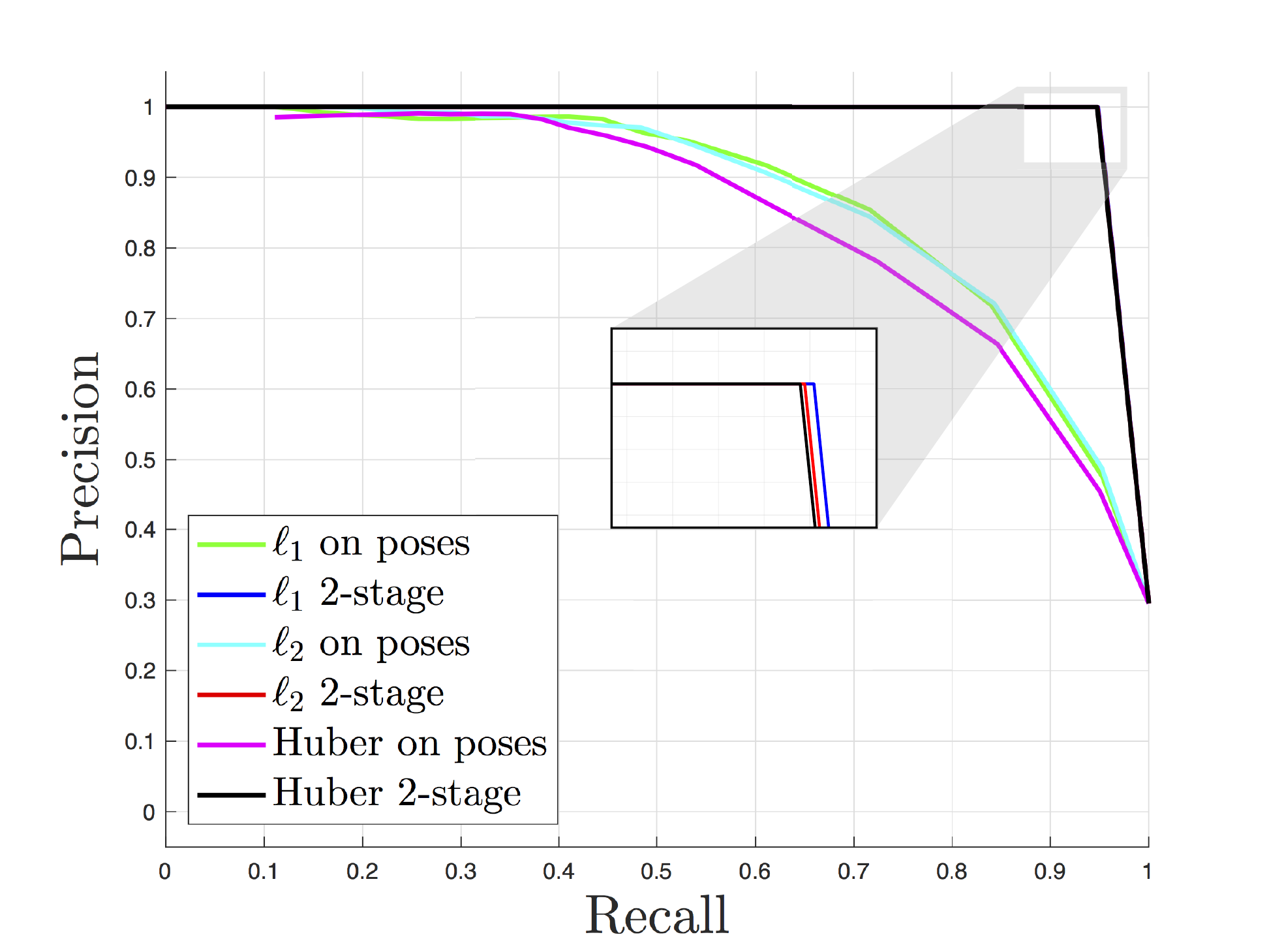}  \\ 
\centering (b) {\footnotesize $n=50$.}
\end{minipage}%
\end{tabular}
\end{minipage}
\caption{Precision/Recall curves for outliers detection based on the residual errors 
corresponding to the 6 techniques discussed in this paper.
\label{fig:PRcurves}}
\vspace{-0.3cm}
\end{figure}


\subsection{Comparison against the State of the Art}
\label{sec:gridAndComparisons}

In this section we show that the proposed approaches outperform the \emph{Dynamic Covariance Scaling} approach~\cite{Agarwal13icra}.

We consider a {\itshape Grid} graph, i.e., a Manhattan world model similar to the one used in related work~\cite{Carlone16tro-duality2D,Sunderhauf13icra}; \red{we simulate a complete grid, as shown in Fig.~\ref{fig:graphs}(c)}. The ground truth poses and measurements associated to the nodes and edges of the grid graph are computed 
as described in Section~\ref{sec:setup}; the {\itshape Grid} graph has an odometric path connecting all nodes.
Each node in the Grid has at most 4 neighbors (corresponding to adjacent nodes in the grid) hence it is less connected than 
the graph used in Section~\ref{sec:setup}, and closer to pose graphs found in practice.
For the comparison in this section, we focus on the 2-stage approaches, since we already observed 
that 
they ensure the best performance. 
We benchmark the proposed techniques against \gtwoo~\cite{Kuemmerle11icra}, a non-robust \PGO solver based on Gauss-Newton 
optimization, and \emph{Dynamic Covariance Scaling}~\cite{Agarwal13icra} (\dcs), a state-of-the-art robust \PGO approach.
Both approaches need an initial guess: in our tests we compute the initial guess from the odometric edges, 
following common practice. We use default parameters for both \gtwoo and \dcs.

Fig.~\ref{fig:aveDistFromGT_grid} shows the mean translation and rotation errors, averaged over 30 runs, 
 for increasing amount of outliers, and for Grid graphs with 20 and 50 nodes. 
Similarly to what we observed in the previous sections, all 2-stage approaches ensure similar performance. 
Moreover, they dominate \gtwoo and \dcs in all test instances. 
\gtwoo is not a robust solver, hence this result is expected. On the other hand, we observe that also 
\dcs has degraded performance: in our tests the initial guess to \dcs is computed from the odometry, and 
some of the odometric edges may be spurious; in those cases, \dcs starts from an incorrect initialization 
and its robust cost has the undesirable effect of ``disabling'' the correct edges which are not consistent with the 
initial guess. This further stresses the importance of designing global techniques for robust \PGO that 
do not rely on an initial guess. 
On the downside, we observe that while the proposed techniques outperform \dcs, their performance 
is significantly worse than the one observed in Section~\ref{sec:robustPGO}. 
This suggests that the performance of the proposed convex relaxations degrades for graphs with 
low connectivity. 

\newcommand{\myscale}{0.26}

\begin{figure}[h]
\begin{minipage}{\columnwidth}
\begin{tabular}{cc}
\hspace{-5 mm}
\begin{minipage}{0.5\columnwidth}%
\includegraphics[scale=\myscale]{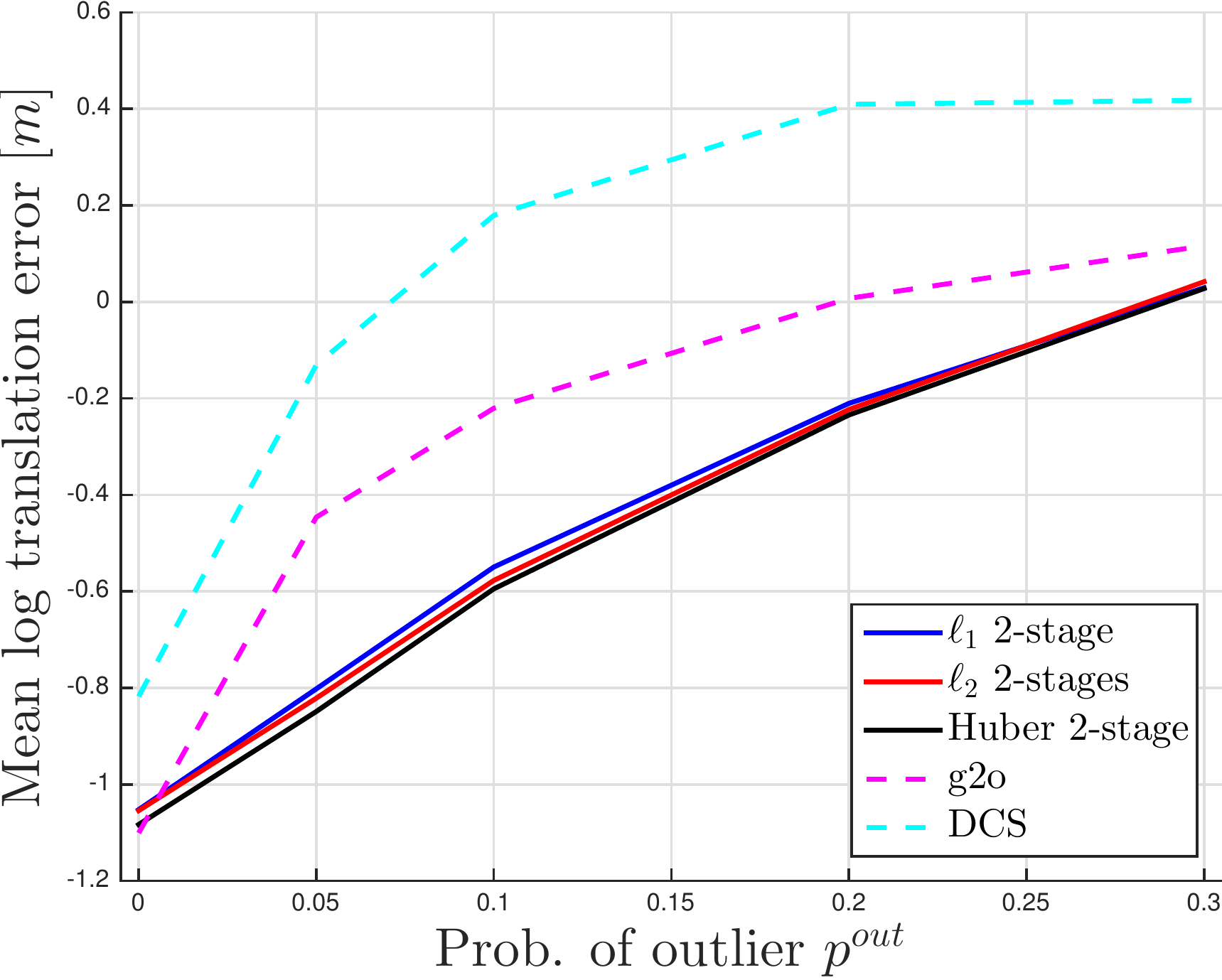}  \\ 
\end{minipage}%
&
\hspace{-3 mm}
\begin{minipage}{0.5\columnwidth}%
\centering
\vspace{-3.5mm}
\includegraphics[scale=\myscale]{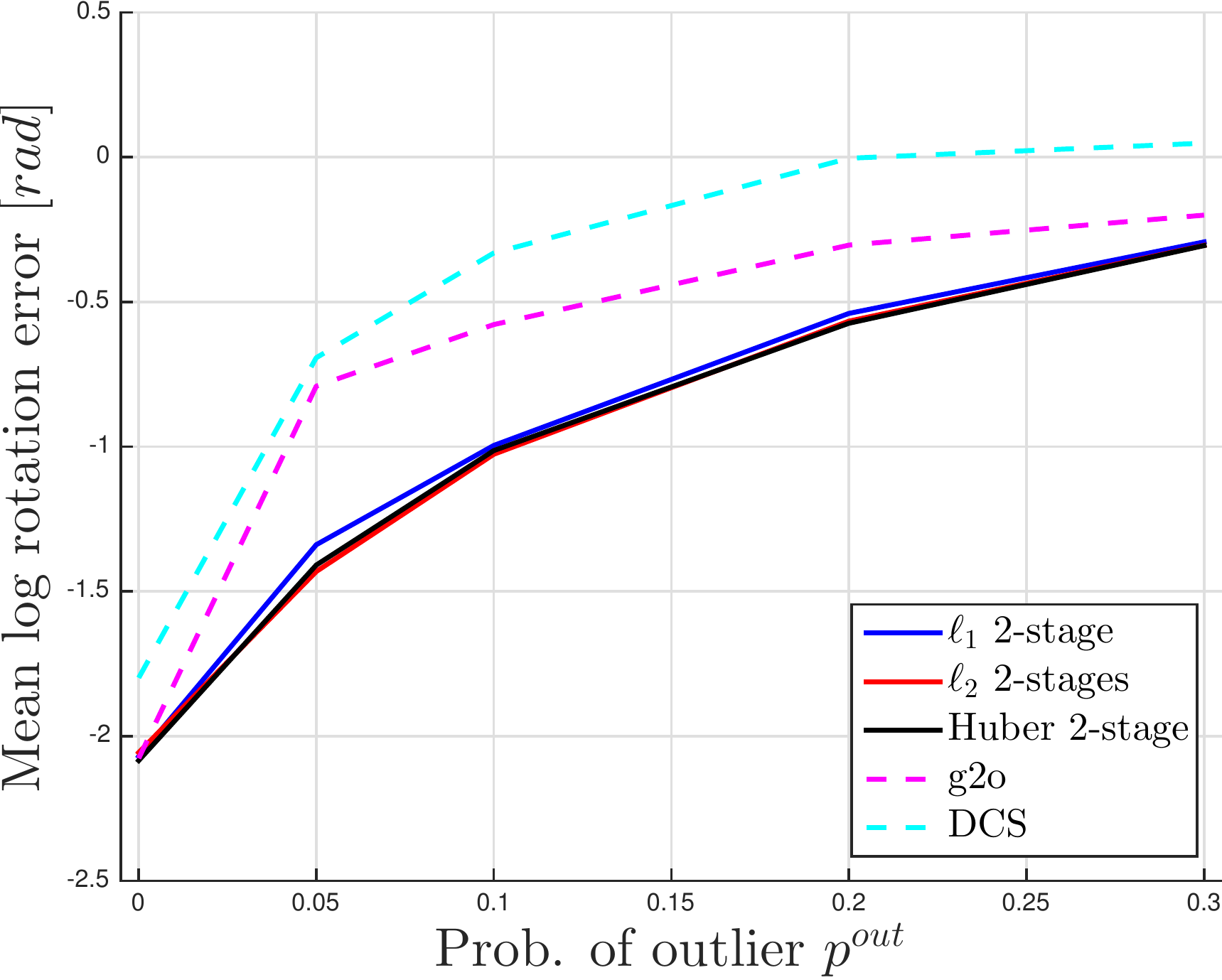}  \\ 
\end{minipage}%
\end{tabular} 
\vspace{-0.5cm} \\
\centering
(a) {\footnotesize  {\itshape Grid}, $n=20$.}
\end{minipage}%
\\
\begin{minipage}{\columnwidth}
\begin{tabular}{cc}
\hspace{-5 mm}
\begin{minipage}{0.5\columnwidth}%
\includegraphics[scale=\myscale]{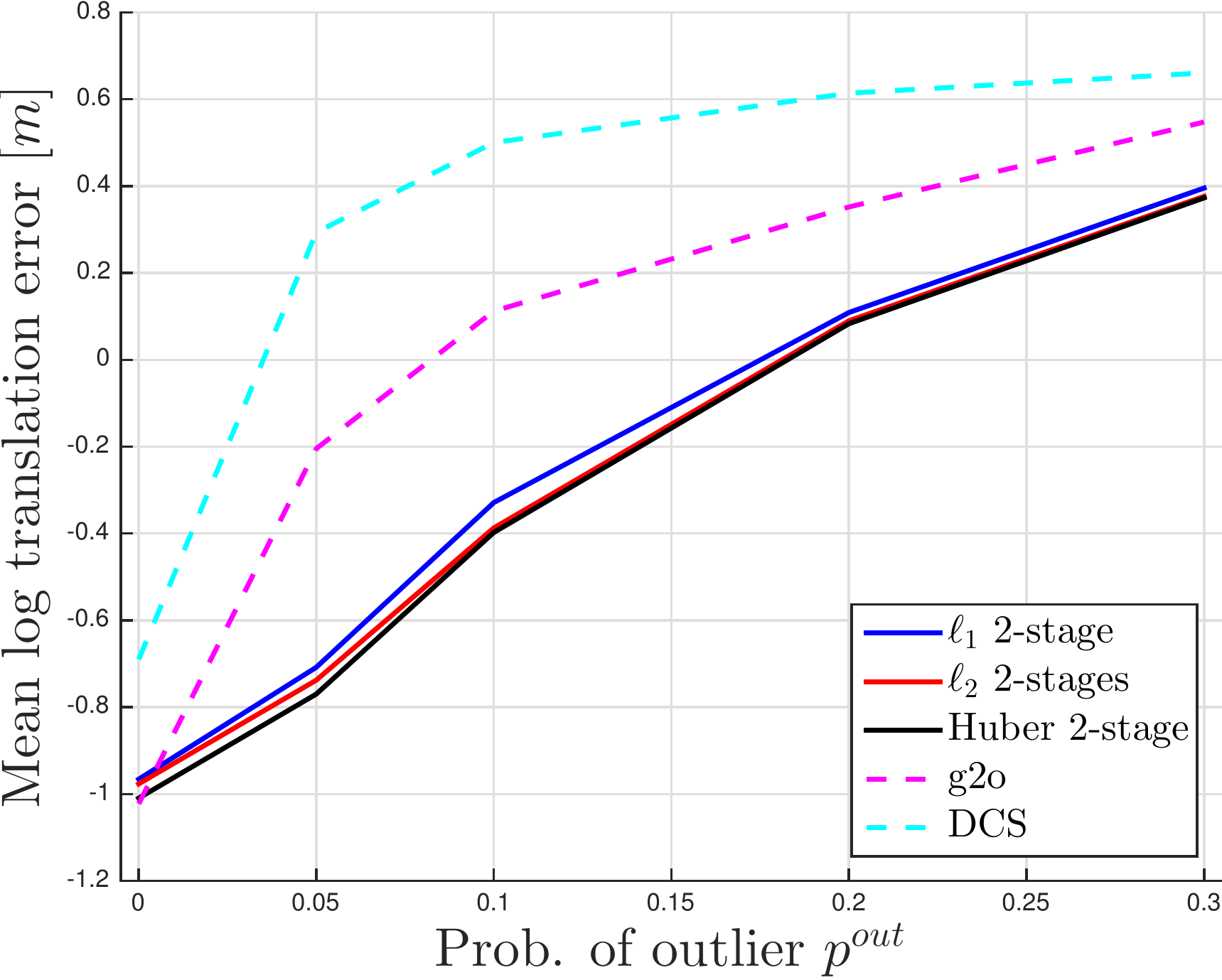}  \\ 
\end{minipage}%
&
\hspace{-3 mm}
\begin{minipage}{0.5\columnwidth}%
\centering
\vspace{-3.5mm}
\includegraphics[scale=\myscale]{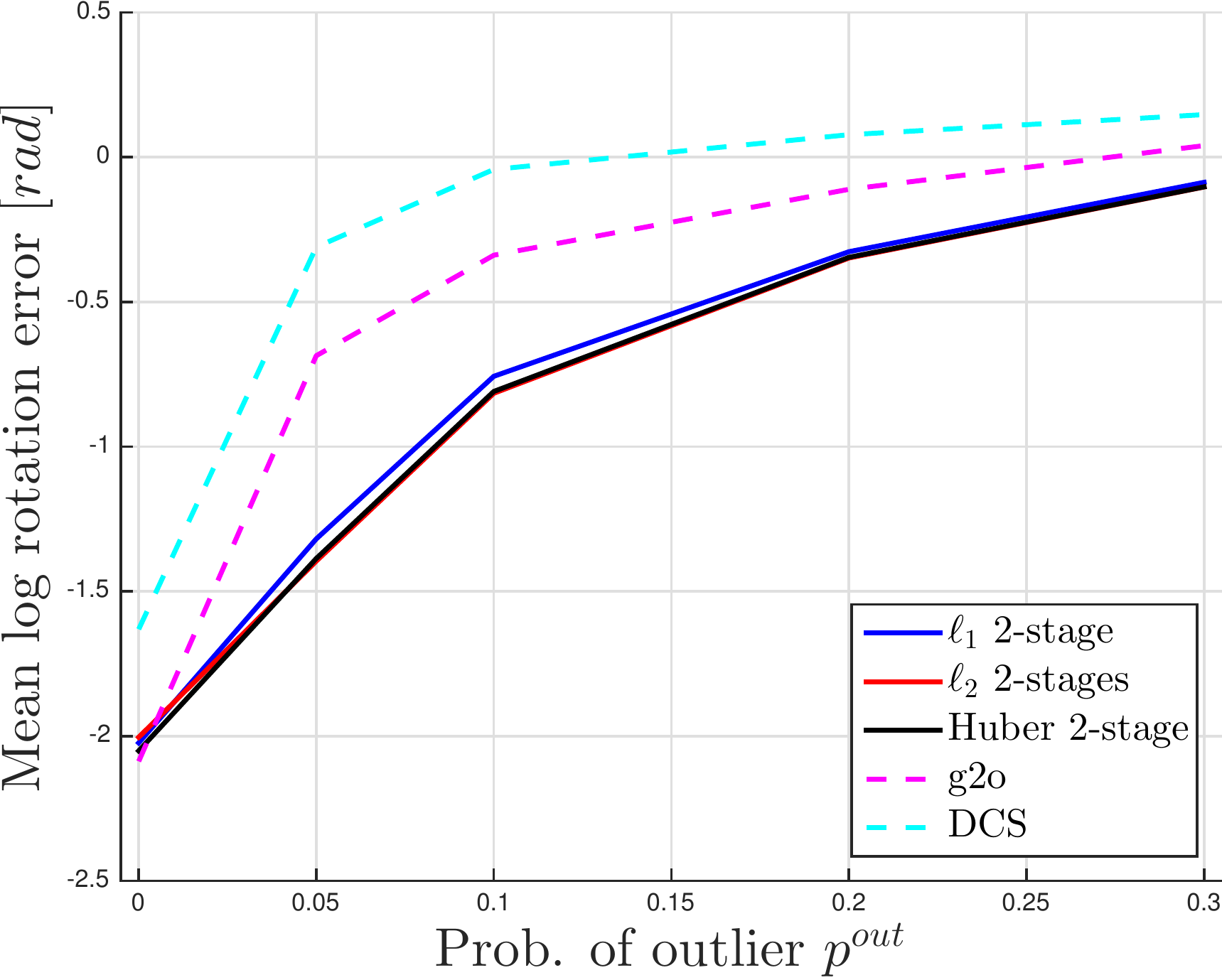}   \\ 
\end{minipage}%
\end{tabular}
\vspace{-0.5cm} \\
\centering
(b) {\footnotesize  {\itshape Grid}, $n=50$.} 
\end{minipage}%
\caption{Estimation errors for the 2-stage approaches proposed in this paper and for \gtwoo and \dcs.
Left column: average translation error. Right column: average rotation error. Grid graphs with (a) n = 20 and (b) n = 50 nodes.
\label{fig:aveDistFromGT_grid}}
\vspace{-4mm}
\end{figure}


\section{Conclusion}
\label{sec:9}
We proposed robust approaches for the solution of the PGO problem in presence of outliers.
We considered three robust cost function (least unsquared deviation, least absolute deviation, and Huber loss) 
and we provided a systematic way of computing a convex relaxation of the resulting (nonconvex) optimization problems 
(1-stage techniques). For each technique we also provided an alternative 2-stage solver, which decouples rotation and 
translation estimation (2-stage techniques), as well as conditions under which the relaxation is exact.
Numerical experiments suggest that while the relaxation implies a loss in accuracy in 1-stage techniques, 
the 2-stage techniques ensure accurate estimation in the face of many outliers, \red{particularly when tested on well-connected graphs. 
This also enables outlier identification and failure detection.} 

This work opens several avenues for future research, including the 
analysis of 3D problem instances, the study of the influence of the graph 
topology on the performance of the proposed convex relaxations, and 
a deeper theoretical investigation of the advantage of 2-stage techniques 
over 1-stage approaches. Future work also includes the investigation of 
ad-hoc numerical solvers to efficiently solve the SDP relaxations 
discussed in this paper: while
the general purpose SDP solver used 
in our current implementation scales poorly in the problem size 
(\red{it requires tens of seconds  to solve an instance with 50 poses}),  
our ultimate goal is to solve large PGO instances ($>1000$ poses) efficiently and robustly.

\vspace{0.1cm}
{\bf Acknowledgments.} The authors gratefully acknowledge Carlo Tommolillo for his help with the numerical evaluation.


\bibliographystyle{plain}
\bibliography{carlo,refs,myRefs}

\end{document}